\documentclass{article}

\usepackage{natbib}
\usepackage[margin=1in]{geometry}

\newif\iftechreport
\newif\iftwocolumns
\newif\ifincludebody
\newif\ifincludeappendix
\newif\ifincludeinlineproofs
\newif\ificmlformat
\newif\ifspacetight

\techreporttrue
\twocolumnsfalse
\includebodytrue
\includeinlineproofstrue
\includeappendixtrue
\icmlformatfalse
\spacetightfalse

\usepackage{times}
\usepackage{amsmath}
\usepackage{amssymb}
\usepackage{amsthm}
\usepackage{thmtools}
\usepackage{thm-restate}
\usepackage{enumerate}
\usepackage{graphicx}
\usepackage{subfigure}
\usepackage{pdfsync}
\usepackage{comment}
\usepackage{xspace}

\usepackage{maz}

\def\internal{\text{\rm internal}}
\def\swap{\text{\rm swap}}
\def\external{\text{\rm external}}
\def\localinternal{\text{\rm localinternal}}
\def\localswap{\text{\rm localswap}}
\def\localexternal{\text{\rm localexternal}}
\def\localcolor{\text{\rm localcolor}}
\def\dist{\text{\rm d}}
\def\root{\text{\rm root}}
\def\level{{\cal L}}
\def\C{\mathbf{C}}
\def\Ecloserto#1{E^{#1}}
\def\Root{\root}
\def\Level{\level}

\newcommand{\TT}{\mbox{true}}
\newcommand{\FA}{\mbox{false}}

\makeatletter
\newcommand\dropbox[1]{\setbox\@tempboxa\vbox to\ht\strutbox{\hbox{#1}\vss}%
  \dp\@tempboxa\dp\strutbox\box\@tempboxa}
\newcommand\basebox[1]{\setbox\@tempboxa\vbox{\hbox{#1}}%
  \setlength\@tempdima{\ht\@tempboxa}\addtolength\@tempdima{\dp\@tempboxa}%
  \ht\@tempboxa=0.5\@tempdima\dp\@tempboxa=0.5\@tempdima%
  \box\@tempboxa}
\makeatother

\ifspacetight
\let\oldnormalsize\normalsize
\def\normalsize{\oldnormalsize%
  \abovedisplayskip2pt plus1pt minus 2pt%
  \belowdisplayskip2pt plus1pt minus 2pt}
\makeatletter
\renewcommand\paragraph[1]{\par\noindent{\bf #1.}\xspace}
\renewcommand\subsection[1]{\paragraph{#1}}
\makeatother
\fi

\def\viz{\mbox{viz.}}

\declaretheorem{requirement}
\declaretheorem{example}

\ificmlformat
\icmltitlerunning{On Local Regret}
\else
\title{On Local Regret}
\author{Michael Bowling\\
Computing Science Department\\
University of Alberta \\
Edmonton, Alberta T6G2E8 Canada \\
\texttt{\small bowling@cs.ualberta.ca}
\and 
Martin Zinkevich \\
Yahoo! Research \\
Santa Clara, CA 95051 USA \\
\texttt{\small maz@yahoo-inc.com}
}
\fi

\begin{document}

\ificmlformat
\twocolumn[
\icmltitle{On Local Regret}

\icmlauthor{Michael Bowling}{bowling@cs.ualberta.ca}
\icmladdress{Computing Science Department, University of Alberta, Edmonton, Alberta T6G2E8 Canada}
\icmlauthor{Martin Zinkevich}{maz@yahoo-inc.com}
\icmladdress{Yahoo! Research, Santa Clara, CA 95051 USA}

\vskip 0.3in
]
\else
\maketitle
\fi

\ifincludebody
\begin{abstract}
Online learning aims to perform nearly as well as the best hypothesis in hindsight.  For some hypothesis classes, though, even finding the best hypothesis offline is challenging.  In such offline cases, local search techniques are often employed and only local optimality guaranteed.  For online decision-making with such hypothesis classes, we introduce local regret, a generalization of regret that aims to perform nearly as well as only nearby hypotheses.  We then present a general algorithm to minimize local regret with arbitrary locality graphs.  We also show how the graph structure can be exploited to drastically speed learning.  These algorithms are then demonstrated on a diverse set of online problems: online disjunct learning, online Max-SAT, and online decision tree learning.
\end{abstract}

\section{Introduction}
An online learning task involves repeatedly taking actions and, after an action is chosen, observing the result of that action.  This is in contrast to offline learning where the decisions are made based on a fixed batch of training data.  As a consequence offline learning typically requires i.i.d. assumptions about how the results of actions are generated (on the training data, and all future data).  In online learning, no such assumptions are required.  Instead, the metric of performance used is regret: the amount of additional utility that could have been gained if some alternative sequence of actions had been chosen.  The set of alternative sequences that are considered defines the notion of regret.  Regret is more than just a measure of performance, though, it also guides algorithms.  For specific notions of regret, no-regret algorithms exist, for which the total regret is growing at worst sublinearly with time, hence their average regret goes to zero.  These guarantees can be made with no i.i.d., or equivalent assumption, on the results of the actions.

One traditional drawback of regret concepts is that the number of alternatives considered must be finite.  This is typically achieved by assuming the number of available actions is finite, and for practical purposes, small.  In offline learning this is not at all the case: offline hypothesis classes are usually very large, if not infinite.  There have been attempts to achieve regret guarantees for infinite action spaces, but these have all required assumptions to be made on the action outcomes (e.g., convexity or smoothness).  In this work, we propose new notions of regret, specifically for very large or infinite action sets, while avoiding any significant assumptions on the sequence of action outcomes.  Instead, the action set is assumed to come equipped with a notion of locality, and regret is redefined to respect this notion of locality.  This approach allows the online paradigm with its style of regret guarantees to be applied to previously intractable tasks and hypothesis classes.

\section{Background}
For $t \in \{ 1, 2, \ldots \}$, let $a^t \in A$ be the action at time $t$, and $u^t : A \rightarrow \mathbb{R}$ be the utility function over actions at time $t$.
\begin{requirement}
For all $t$, $\max_{a,b \in A} |u^t(a) - u^t(b)| \le \Delta$.
\label{req:bounded-u}
\end{requirement}
The basic building block of regret is the additional utility that could have been gained if some action $b$ was chosen in place of action $a$: $R^T_{a,b} = \sum_{t=1}^T 1(a^t = a) \left( u^t(b) - u^t(a) \right)$,
where $1(\text{\it condition})$ is equal to $1$ when {\it condition} is true and $0$ otherwise.  We can use this building block to define the traditional notions of regret.
\begin{gather}
R^T_{\internal} = \max_{a,b \in A} R^{T,+}_{a,b} \quad\quad
R^T_{\swap} = \sum_{a \in A} \max_{b \in A} R^{T,+}_{a,b} \\
R^T_{\external} = \max_{b \in A} 
  \left(\ifspacetight\textstyle\fi\sum_{a \in A} R^T_{a,b}\right)^+ 
\end{gather}
where $x^+ = \max(x, 0)$ so that $R^{T,+}_{a,b} = \max(R^T_{a,b}, 0)$.  Internal regret~\citep{HarMas02} is the maximum utility that could be gained if one action had been chosen in place of some other action.  Swap regret~\citep{GreJaf03} is the maximum utility gained if each action could be replaced by another.  External regret~\citep{Hannan57}, which is the original pioneering concept of regret, is the maximum utility gained by replacing all actions with one particular action.  This is the most relaxed of the three concepts, and while the others must concern themselves with $|A|^2$ possible regret values (for all pairs of actions) external regret only need worry about $|A|$ regret values.  So although the guarantee is weaker, it is a simpler concept to learn which can make it considerably more attractive.  These three regret notions have the following relationships.
\begin{align}
R^T_{\internal} &\le R^T_{\swap} \le |A| R^T_{\internal} &
R^T_{\external} &\le R^T_{\swap}
\end{align}

\subsection{Infinite Action Spaces}
This paper considers situations where $A$ is infinite.  To keep the notation simple, we will use max operations over actions to mean suprema operations and summations over actions to mean the suprema of the sum over all finite subsets of actions.  Since we will be focused on regret over a finite time period, there will only ever be a finite set of actually selected actions and, hence only a finite number of non-zero regrets, $R^T_{a,b}$.  The summations over actions will always be thought to be restricted to this finite set.

None of the three traditional regret concepts are well-suited to $A$ being infinite.  Not only does $|A|$ appear in the regret bounds, but one can demonstrate that it is impossible to have no regret in some infinite cases.  Consider $A = \mathbb{N}$ and let $u^t$ be a step function, so
$u^t(a) = 1$ if $a > y^t$ for some $y^t$ and $0$ otherwise.  Imagine $y^t$ is selected so that $\Pr[ a^t > y^t | u^{1,\ldots,T-1}, a^{1,\ldots,T-1}]\le 0.001$, which is always possible.  Essentially, high utility is always just beyond the largest action selected.  Now, consider $y^* = 1 + \max_{t \le T} y^t$.  In expectation $\frac{1}{T} \sum_{t=1}^T u^t(a_t) \leq 0.001$ while $\frac{1}{T}\sum_{t=1}^T u^t(y^*) = 1$ (\ie, there is large internal and external regret for not having played $y^*$,) so the average regret cannot approach zero.

Most attempts to handle infinite action spaces have proceeded by making assumptions on both $A$ and $u$.  For example, if $A$ is a compact, convex subset of $\mathbb{R}^n$ and the utilities are convex with bounded gradient on $A$, then you can minimize regret even though $A$ is infinite~\citep{Zin03}.
We take an alternative approach where we make use of a notion of locality on the set $A$, and modify regret concepts to respect this locality.  Different notions of locality then result in different notions of regret.  Although this typically results in a weaker form of regret for finite sets, it breaks all dependence of regret on the size of $A$ and allows it to even be applied when $A$ is infinite and $u$ is an arbitrary (although still bounded) function.  Wide range regret
methods~\cite{Lehrer03} can also bound regret with respect to a set of (countably) infinite ``alternatives'', but unlike our results, their asymptotic bound does not apply uniformly across the set, and uniform finite-time bounds depend upon a finite action space~\cite{BlumMansour07}.

\section{Local Regret Concepts}

Let $G = (V, E)$ be a directed graph on the set of actions, \ie, $V = A$.  We do not assume $A$ is finite, but we do assume $G$ has bounded out-degree $D = \max_{a\in V} |\{ b: (a,b) \in E\}|$.  This graph can be viewed as defining a notion of locality.  The semantics of an edge from $a$ to $b$ is that one should consider possibly taking action $b$ in place of action $a$.  Or rather, if there is no edge from $a$ to $b$ then one need not have any regret for not having taken action $b$ when $a$ was taken.  By limiting regret only to the edges in this graph, we get the notion of local regret.  Just as with traditional regret, which we will now refer to as global regret, we can define different variants of regret.
\begin{align}
R^T_{\localinternal} &= \max_{(a,b) \in E} R^{T,+}_{a,b} 
\iftwocolumns \\ \else & \fi
R^T_{\localswap} &= \sum_{a \in A} \max_{b : (a,b)\in E} R^{T,+}_{a,b} 
\end{align}
Local internal and local swap regret just involve limiting regret to edges in $G$.  Local external regret is more subtle and requires a notion of edge lengths.  For all edges $(i,j)\in E$, let $c(i,j) > 0$ be the edge's positive length.  Define $\dist(a,b)$ to be the sum of the edge lengths on a shortest path from vertex $a$ to vertex $b$, and $\Ecloserto{b} = \{ (i, j) \in E : d(i, j) = c(i, j) + d(j, b) \}$ to be the set of edges that are on any shortest path to vertex $b$.
\begin{equation}
R^T_{\localexternal} = \max_{b \in A} \left(\sum_{(i,j) \in \Ecloserto{b}} R^T_{i,j}/D\right)^+
\end{equation}
Global external regret considers changing all actions to some target action, regardless of locality or distance between the actions.
In local external regret, only adjacent actions are considered, and so actions are only replaced with actions that take one step toward the target action.  The factor of $1/D$ scales the regret of any one action by the out-degree, which is the maximum number of actions that could be one-step along a shortest path.  This keeps local external regret on the same scale as local swap regret.

It is easy to see that these concepts hold the same relationships between each other as their global counterparts.
\begin{align}
R^T_{\localinternal} &\le R^T_{\localswap} \le |A| R^T_{\localinternal} \\
R^T_{\localexternal} &\le R^T_{\localswap} 
\end{align}
More interestingly, in complete graphs where there is an edge between every pair of actions (all with unit lengths) and so everything is local, we can exactly equate global and local regret.
\begin{theorem}
If $G$ is a complete graph with unit edge lengths then,
\iftwocolumns
$R^T_{\localinternal} = R^T_{\internal}$; $R^T_{\localswap} = R^T_{\swap}$; and $R^T_{\localexternal} = R^T_{\external}/D$.
\else
\begin{align}
R^T_{\localinternal} &= R^T_{\internal} & R^T_{\localswap} &= R^T_{\swap} & \mbox{and} &&
R^T_{\localexternal} = R^T_{\external}/D.
\end{align}
\fi
\end{theorem}

\ifincludeinlineproofs
\begin{proof}
\begin{align}
R^T_{\localinternal} &= \max_{(a,b) \in E} R^{T,+}_{a,b} = \max_{a,b \in A} R^{T,+}_{a,b} = R^T_{\internal} \\
R^T_{\localswap} &= \sum_{a \in A} \max_{b : (a,b)\in E} R^{T,+}_{a,b} = 
\sum_{a \in A} \max_{b \in A} R^{T,+}_{a,b} = R^T_{\swap} \\
R^T_{\localexternal} &= \max_{b \in A} \sum_{(i,j) \in \Ecloserto{b}} R^{T,+}_{i,j} / D \\
&= 1/D \max_{b \in A} \sum_{a \in A} R^{T,+}_{a,b} = R^T_{\external} / D
\end{align}
\end{proof}
\else
The proofs of the paper's theorems are not included for space reasons.  When there is a useful insight, we discuss the proof techniques and implications.  The full proofs can be found in the longer version of this work available as a technical report~\cite{12icml-localregret-tr}. 
\fi

So our concepts of local regret match up with global regret when the graph is complete.  Of course, we are not really interested in complete graphs, but rather more intricate locality structures with a large or infinite number of vertices, but a small out-degree.  Before going on to present algorithms for minimizing local regret, we consider possible graphs for three different online decision tasks to illustrate where the graphs come from and what form they might take.

\begin{example}[Online Max-3SAT]\rm
\label{ex:max-cnf} 
Consider an online version of Max-3SAT.  The task is to choose an assignment for $n$ boolean variables: $A = \{0, 1\}^n$.  After an assignment is chosen a clause is observed; the utility is 1 if the clause is satisfied by the chosen assignment, 0 otherwise.  Note that $|A| = 2^n$ which is computationally intractable for global regret concepts if $n$ is even moderately large.  One possible locality graph for this hypothesis class is the hypercube with an edge from $a$ to $b$ if and only if $a$ and $b$ differ on the assignment of exactly one variable\ifspacetight\else~(see Figure~\ref{fig:graphs:hypercube})\fi, and all edges have unit lengths.  So the out-degree $D$ for this graph is only $n$.  Local regret, then, corresponds to the regret for not having changed the assignment of just one variable.  In essence, minimizing this concept of regret is the online equivalent of local search (e.g.,  WalkSAT~\citep{SeKaCo93}) on the maximum satisfiability problem, an offline task where all of the clauses are known up front.
\end{example}
\ifspacetight\else
\iftwocolumns
\def\figscale{0.34}
\begin{figure}
\centering
\includegraphics[scale=\figscale]{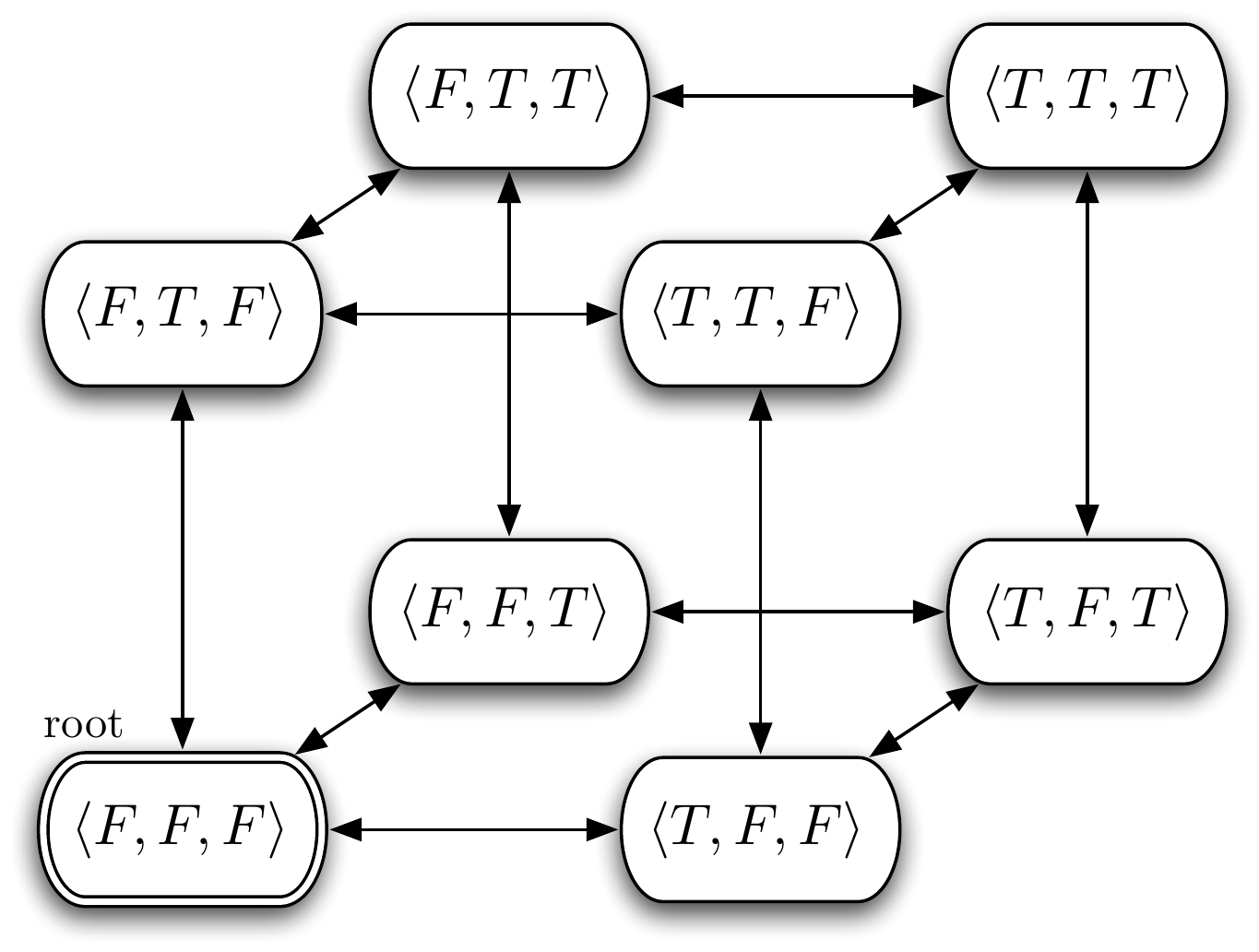}
\caption{Example graph for Max-3SAT and disjuncts ($n=3$).}%
\label{fig:graphs:hypercube}
\end{figure}
\else
\begin{figure}
\centering
\subfigure[]{%
\includegraphics[width=0.42\hsize]{figs/hypercube}
\label{fig:graphs:hypercube}}
\subfigure[]{%
\includegraphics[width=0.55\hsize]{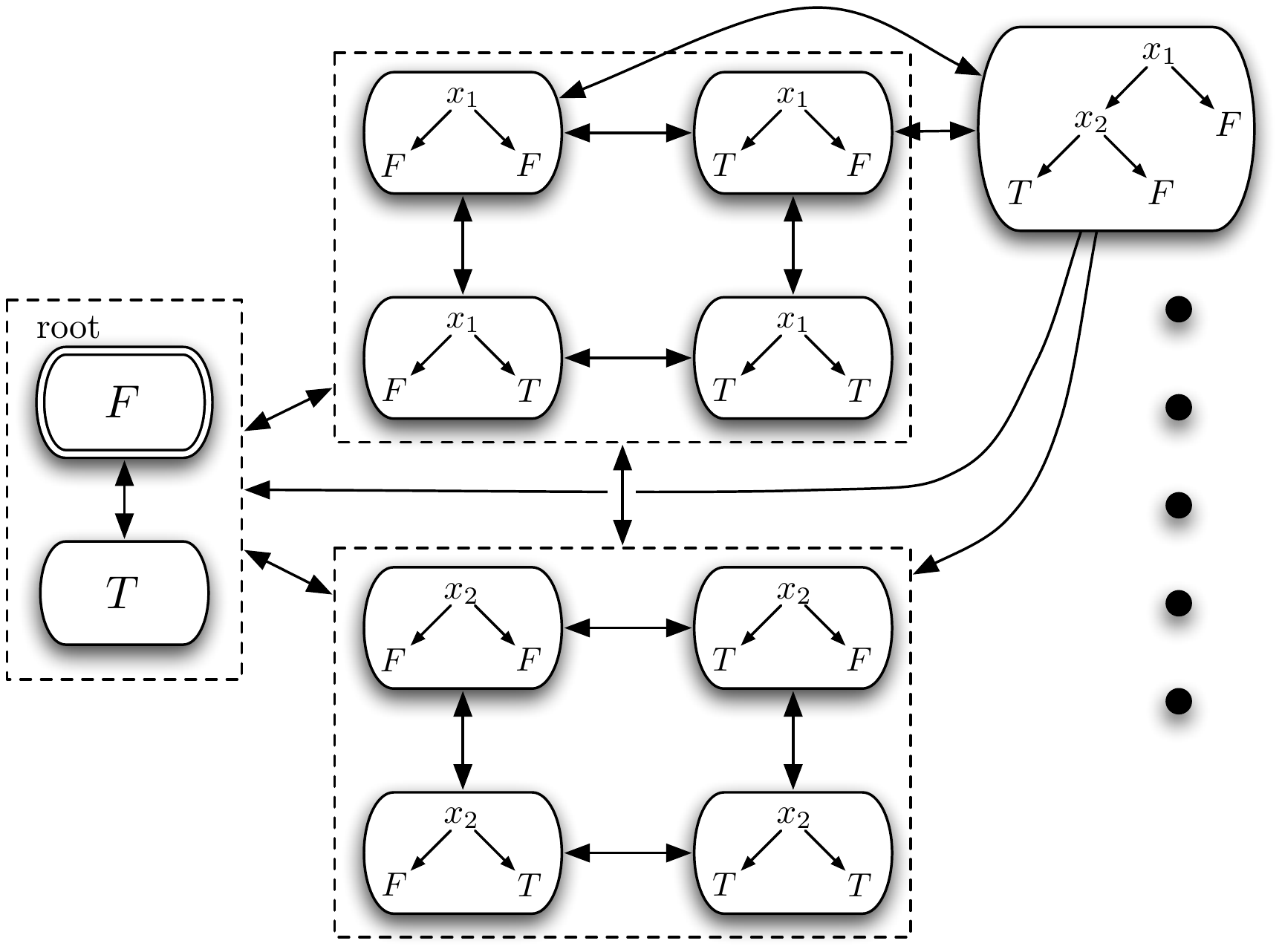}
\label{fig:graphs:dtree}}
\caption{Example graphs.  (a) Graph for Max-3SAT and disjuncts ($n=3$).  (b) Part of graph for decision trees ($n=2$), where edges to and from the dashed boxes represent edges to and from every vertex in the box.}
\end{figure}
\fi
\fi
\begin{example}[Online Disjunct Learning]\rm
\label{ex:disjuncts} 
Consider a boolean online classification task where input features are boolean vectors $x \in \{0,1\}^n$ and the target $y$ is also boolean.  Consider $A = \{0,1\}^n$, to be the set of all disjuncts such that $a \in A$ corresponds to the disjunct $x_{i_1} \vee x_{i_2} \vee \ldots \vee x_{i_k}$ where $i_{1\le j\le k}$ are all of the $k$ indices of $a$ such that $a_{i_j} = 1$.  In this online task, one must repeatedly choose a disjunct and then observe an instance which includes a feature vector and the correct response.  There is a utility of 1 if the chosen disjunct over the feature vector results in the correct response; 0 otherwise.   Although a very different task, the action space $A = \{0, 1\}^n$ is the same as with Online Max-SAT and we can consider the same locality structure as that proposed for disjuncts: a hypercube with unit length edges for adding or removing a single variable to the disjunction\ifspacetight\else~(see Figure~\ref{fig:graphs:hypercube})\fi.  
And as before $|A| = 2^n$ while $D = n$.  
\end{example}

\begin{example}[Online Decision Tree Learning]\rm
\label{ex:dtree} 
Imagine the same boolean online classification task for learning disjuncts, but the hypothesis class is the set of all possible decision trees.  The number of possible decision trees for $n$ boolean variables is more than a staggering $2^{2^n}$, which for any practical purpose is infinite.  We can construct a graph structure that mimics the way decision trees are typically constructed offline, such as with C4.5~\citep{Quinlan93}.  In the graph $G$, add an edge from one decision tree to another if and only if the latter can be constructed by choosing any node (internal or leaf) of the former and replacing the subtree rooted at the node with a decision stump or a label.  There is one exception: you cannot replace a non-leaf subtree with a stump splitting on the same variable as that of the root of the subtree.  \ifspacetight\else See Figure~\ref{fig:graphs:dtree} for a portion of the graph. \fi  Edges that replace a subtree with a label have length 1, while edges replacing a subtree with a stump (being a more complex change) have distance 1.1.  So, we have local regret for not having further refined a leaf or collapsing a subtree to a simpler stump or leaf.  Notice that the graph edges in this case are not all symmetric (\viz, collapsing edges).  In essence, this is the online equivalent of tree splitting algorithms.  While $|A| \ge 2^{2^n}$, the out-degree is no more than $(n+1)2^{n+1}$.  The maximum size of the out-degree still appears disconcertingly large, and we will return to this issue in Section~\ref{sec:exploiting-locality} where we show how we can exploit the graph structure to further simplify learning.
\end{example}

\ifspacetight\else
\iftwocolumns
\begin{figure}
\centering
\includegraphics[scale=\figscale]{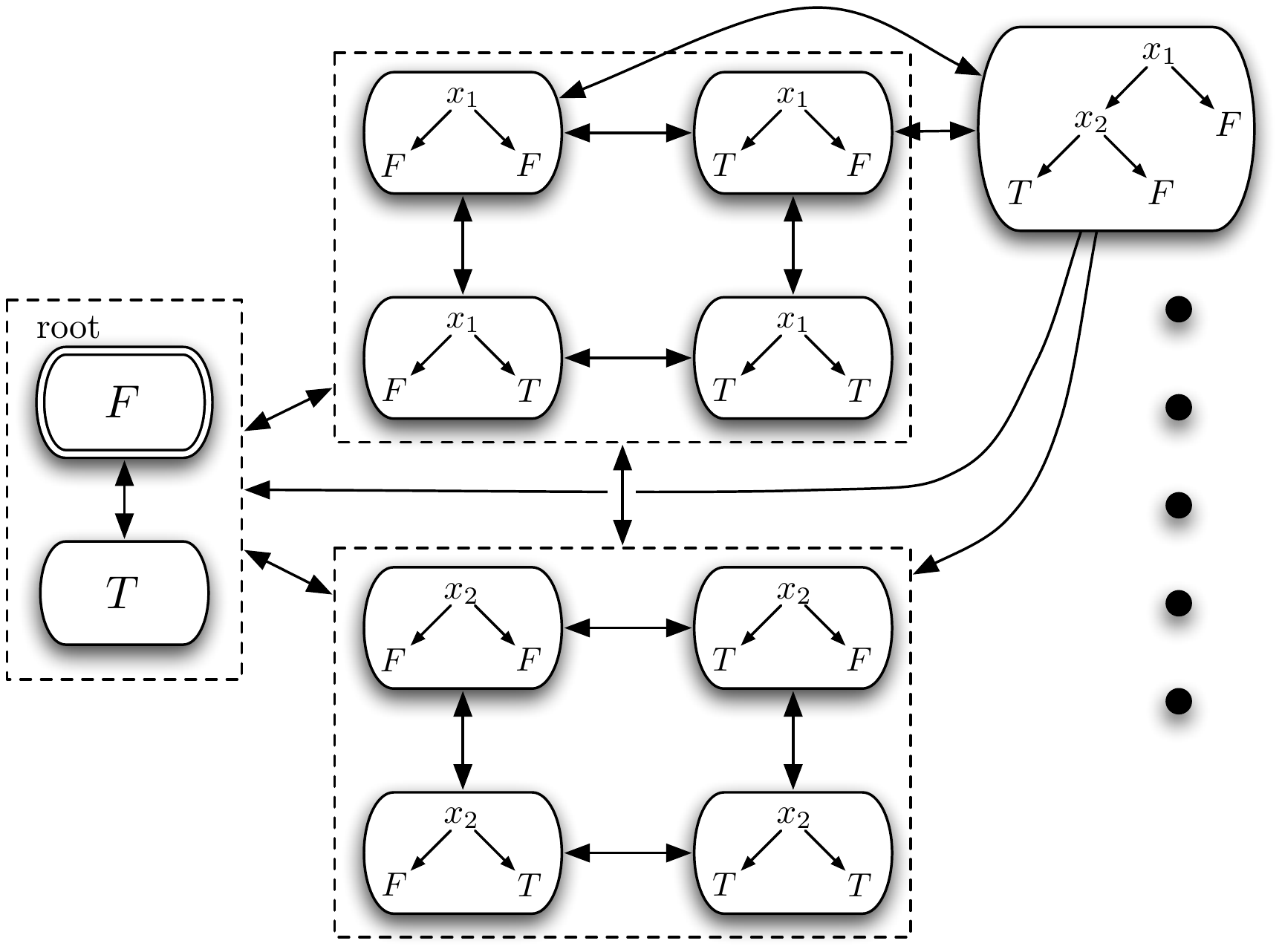}
\caption{Part of an example graph for decision trees ($n=2$).}
\label{fig:graphs:dtree}
\end{figure}
\fi
\fi

\section{An Algorithm for Local Swap Regret}
\label{sec:analgorothmforlocalswapregret}
We now present an algorithm for minimizing local swap regret, similar to global swap regret algorithms~\citep{HarMas02,GreJaf03}, but with substantial differences.  The algorithm essentially chooses actions according to the stationary distribution of a Markov process on the graph, with the transition probabilities on the edges being proportional to the accumulated regrets.  However there are two caveats that are needed for it to handle infinite graphs: it is prevented from playing beyond a particular distance from a designated $\root$ vertex, and there is an internal bias towards the actual actions chosen.  

Formally, let $\root$ be some designated vertex.  Define $\dist_1$ to be the unweighted shortest path distance between two vertices. Define the level of a vertex as its distance from $\root$: $\level(v) = \dist_1(\root, v)$.  Note that, $\level(\root) = 0$, and $\forall (i,j) \in E$, $\level(j) \le \level(i) + 1$.
All of the algorithms in this paper take a parameter $L$, and will never choose actions at a level greater than $L$.  In addition, the algorithms all maintain values $\tilde{R}^t_{i,j}$ (which are biased versions of $R^t_{i,j}$) and use these to compute $\pi^t_j$, the probability of choosing action $j$ at time $t$.  These probabilities are always computed according to the following requirement, which is a generalization of~\citep{HarMas02,GreJaf03}.

\begin{requirement}\raggedright
Given a parameter $L$, for all $t \le T$, and some $\tilde{R}^{t,+}_{i,j}$ let $\pi^{t+1}$ be such that
\begin{enumerate}[(a)]
\item $\sum_{j \in V} \pi^{t+1}_j = 1$, and $\forall j\in V$, $\pi^{t+1}_j \ge 0$
\item $\forall j\in V$ such that $\level(j) > L$, $\pi^{t+1}_j = 0$.
\item $\forall j\in V$ such that $1 \le \level(j) \le L$,
$\pi^{t+1}_j = \sum_{i:(i,j)\in E} (\tilde{R}^{t,+}_{i,j} / M) \pi^{t+1}_i + (1 - \sum_{k:(j,k)\in E} \tilde{R}^{t,+}_{j,k}/M) \pi^{t+1}_j$
\item $\pi^{t+1}_\root = \sum_{i:(i,\root)\in E} (\tilde{R}^{t,+}_{i,\root} / M)\pi^{t+1}_i+\sum_{j : \level(j) = L+1} \sum_{i:(i,j)\in E} (\tilde{R}^{t,+}_{i,j} / M)\pi^{t+1}_i + (1 - \sum_{j : (\root,j)\in E} \tilde{R}^{t,+}_{\root,j} / M) \pi^{t+1}_\root$
\item If there exists $j \in V$ such that $\pi^{t+1}_j > 0$ and $\sum_{k : (j,k) \in E} \tilde{R}^{t,+}_{j,k} = 0$, then for all $j \in V$ where $\pi^{t+1}_j > 0$, $\sum_{k : (j,k) \in E} \tilde{R}^{t,+}_{j,k} = 0$, and we call such a $\pi^{t+1}$ {\em degenerate}.
\end{enumerate}
where $M = \max_{(i,j) \in E} \tilde{R}^{t,+}_{i,j}$.  These conditions require $\pi^{t+1}$ to be the stationary distribution of the transition function whose probabilities on outgoing edges are proportional to their biased positive regret, with the root vertex as the starting state, and all outgoing transitions from vertices in level $L$ going to the root vertex instead.
\label{req:stationary-distribution}
\end{requirement}

\begin{definition}
$(b, L)$-regret matching is the algorithm that initializes $\tilde{R}^0_{i,j} = 0$, chooses actions at time $t$ according to a distribution $\pi^t$ that satisfies Requirement~\ref{req:stationary-distribution} and after choosing action $i$ and observing $u^t$ updates $\tilde{R}^{t}_{i,j} = \tilde{R}^{t-1}_{i,j} + (u^t(j) - u^t(i) - b)$ for all $j$ where $(i,j) \in E$, and for all other $(k,l)\in E$ where $k\ne i$, $\tilde{R}^{t}_{k,l} = \tilde{R}^{t-1}_{k,l}$.
\label{def:alg-local-swap}
\end{definition}
There are two distinguishing factors of our algorithm from~\citep{HarMas02,GreJaf03}: $\tilde{R}\ne R$, and past a certain distance from the root, we loop back.  $\tilde{R}$ differs from $R$ by the bias term, $b$.  This term can be thought of as a bias toward the action selected by the algorithm.  This is \textit{not} the same as approaching the negative orthant with a margin for error.  This small amount is only applied to the action taken, which is very different from adding a small margin of error to \textit{every} edge.

\begin{restatable}{theorem}{thmlocalswap}
For any directed graph with maximum out-degree $D$ and any designated vertex $\root$, $(\Delta/ (L+1), L)$-regret matching, after $T$ steps, will have expected local swap regret no worse than,
\begin{align}
\frac{1}{T} E[R^T_{\mathrm{localswap}}] \le 
\frac{\Delta}{L+1} + \frac{\Delta \sqrt{D |E_L|}}{\sqrt{T}}\label{eqn:localswapbound}
\end{align}
where $E_L = \{(i,j) \in E| \level(i) \le L\}$.  
\label{thm:localswap}
\end{restatable}


\ifincludeappendix
The proof can be found in Appendix~\ref{sec:appendix:local-swap-regret}.
\fi
The overall structure of the proof is similar to~\citep{Blackw56,HarMas02,GreJaf03} with a few significant changes. As with most algorithms based on Blackwell, if there is an action you do not regret taking, playing that action the next round is ``safe''.  If not, the key quantity in the proof is a flow $f_{i,j}=\pi^{t+1}_i \tilde{R}^{t,+}_{i,j}$ for each edge. On most of the graph, the incoming flow is equal to the outgoing flow for each node in levels 1 to $L$. Since all the flow out from the nodes on one level is equal to the flow into the next, the total flow into (and out of) each level is equal. Thus, the flow out of the last level is only $1/(L+1)$ of the total flow on all edges since there are $L+1$ levels, including the root.

Traditionally, we wish to show that the incoming flow of an action times the utility minus the outgoing flow of an action times the utility summed over all nodes is nonpositive, and then Blackwell's condition holds. In traditional proofs, for any given node, the flow in and out are equal, so regardless of the utility, they cancel. For our problem, the flow out of the last level is really a flow into the $(L+1)$st level, not the zeroeth level, so the difference in utilities between the zeroeth level and the $(L+1)$st level creates a problem. On the other hand, because we subtract $b$ from whatever action we select, we get to subtract $b$ times the total flow. Since exactly $1/(L+1)$ fraction of the flow is going into the $(L+1)$st level, these two discrepancies from the traditional approach exactly cancel.  The second term of Equation~\eqref{eqn:localswapbound} is a result of the traditional Blackwell approach.
In the final analysis, we must account for the amount $b$ we subtract from the regret each round. This means that if we get $\tilde{R}$ to approach the negative orthant, we only have $bT$ local swap regret left. This is the first term of Equation~\eqref{eqn:localswapbound}.

\section{Exploiting Locality Structure}

\label{sec:exploiting-locality}

The local swap regret algorithm in the previous section successfully drops all dependence on the size of the action set and thus can be applied even for infinite action sets.  However, the appearance of $|E_L|$ in the bound in Theorem~\ref{thm:localswap} is undesirable as $|E_L| \in O(D^L)$, and $L$ is more likely to be 100 than 2, in order to keep the first term of the bound low.  The bound, therefore, practically provides little beyond an asymptotic guarantee for even the simplest setting of Example~\ref{ex:max-cnf}.  In this section, we will appeal to (i) the structure in the locality graph, and (ii) local external regret to achieve a more practical regret bound and algorithm.

\subsection{Cartesian Product Graphs}
We begin by considering the case of $G$ having a very strong structure, where it can be entirely decomposed into a set of product graphs.  In this case, we can show that by independently minimizing local regret in the product graphs we can minimize local regret in the full graph.

\ifincludeinlineproofs
\begin{theorem}
Let $G$ be a Cartesian product of graphs, $G = G_1 \otimes \ldots \otimes G_k$ where $G_l = (V_l, E_l)$.  For all $l \in \{1, \ldots, k\}$, define $u^t_l : V_l \rightarrow \mathbb{R}$, such that $u^t_l(a_l) = u^t(\left<a^t_1, \ldots, a^t_{l-1}, a_l, a^t_{l+1}, \ldots, a^t_k\right>)$, so $u^t_l$ is a utility function on the $l$th component of the action at time $t$ assuming the other components remain unchanged.
Let $E[l] \subseteq E$ be the set of edges that change only on the $l$th component, so $\{E[l]\}_{l=1,\ldots, k}$ forms a partition of $E$.  Let $D_l \le D$ be the maximum degree of $G_l$.  Finally, define 
\begin{align}
R^{T,l}_{\localexternal} &=
\max_{b \in V_l} \left(\sum_{(i,j) \in E^b_l} \sum_{t=1}^T 1(a^t_l = i) (u^t_l(j) - u^t_l(i)) /D_l\right)^+,
\nonumber
\end{align}
where $\Ecloserto{b}_l = \{ (i, j) \in E[l]  : d(i, b_l) = c(i, j) + d(j, b_l)  \}$, i.e., it contains the edges that moves the $l$th component closer to $b_l$.  Then, $R^T_{\localexternal} \le \sum_{l=1}^k R^{T,l}_{\localexternal}$.
\end{theorem}

\begin{proof}
\begin{align}
R^T_{\localexternal}
&= \max_{b \in V} \left(\sum_{(i,j) \in \Ecloserto{b}} R^T_{i,j}/D\right)^+ \\
&= \max_{b \in V} \left(\sum_{l=1}^k \sum_{(i,j) \in E[l] \cap \Ecloserto{b}} R^T_{i,j} /D\right)^+ \\
&\le \sum_{l=1}^k \max_{b \in V} \left(\sum_{(i,j) \in E[l] \cap \Ecloserto{b}} R^T_{i,j} /D\right)^+ \\
\intertext{Since $D_l \le D$,}
&\le \sum_{l=1}^k \max_{b \in V} \left(\sum_{(i,j) \in E[l] \cap \Ecloserto{b}} R^T_{i,j} /D_l\right)^+ \\
&= \sum_{l=1}^k \max_{b \in V} \left(\sum_{(i,j) \in E[l] \cap \Ecloserto{b}} \sum_{t=1}^T 1(a^t = i) (u^t(j) - u^t(i)) /D_l\right)^+ \\
&= \sum_{l=1}^k \max_{b \in V} \left(\sum_{(i,j) \in \Ecloserto{b}_l} \sum_{t=1}^T 1(a^t = i) (u^t_l(j) - u^t_l(i)) /D_l\right)^+ \\
&= \sum_{l=1}^k R^{T,l}_{\localexternal}
\end{align}
\end{proof}
\else
\begin{theorem}
Let $G$ be a Cartesian product of graphs, $G = G_1 \otimes \ldots \otimes G_k$.  Let $R^{T,l}_\localexternal$ be the measured external regret on the $l^{\rm th}$ component graph, where the action at time $t$ is the $l^{\rm th}$ component of $a^t$ and regret is on the edges in $G_l$ that transform the $l^{\rm th}$ component.  Then, $R^T_{\localexternal} \le \sum_{l=1}^k R^{T,l}_{\localexternal}$.
\end{theorem}
\fi

The implication is that we if we apply independent regret minimization to each factor of our product graph, we can minimize local external regret on the full graph.  For example, consider the hypercube graphs from Example~\ref{ex:max-cnf} and \ref{ex:disjuncts}.  By applying $n$ independent external regret algorithms (the component graphs in this case are 2-vertex complete graphs), the overall local external regret for the graph is at most $n$ times bigger than the factors' regrets, so under regret matching it is bounded by $n\Delta \sqrt{2}/\sqrt{T}$.  Hence, we are able to handle an exponentially large graph (in $n$) with local external regret only growing linearly (in $n$).  If the component graphs are not complete graphs, then we can simply apply our local swap regret algorithm from the previous section to the graph factors, which minimizes local external regret as well.

\subsection{Color Regret}
\label{sec:colorregret}
Cartesian product graphs are a powerful, but not very general structure.  We now substantially generalize the product graph structure, which will allow us to achieve a similar simplification for very general graphs, such as the graph on decision trees in Example~\ref{ex:dtree}.  The key insight of product graphs is that for any vertex $b$, an edge moves toward $b$ if and only if its corresponding edge in its component graph moves toward $b_l$.  In other words, either all of the edges that correspond to some component edge will be included in the external regret sum, or none of the eges will.  We can group together these edges and only worry about the regret of the group and not its constituents.  We generalize this fact to graphs which do not have a product structure. 

\begin{definition}
An edge-coloring $\C = \{C_i \}_{i=1,2,\ldots}$ for an arbitrary graph $G$ with edge lengths is a partition of $E$: $C_i \subseteq E$, $\bigcup_i C_i = E$, and $C_i \bigcap C_j = \emptyset$.  We say that $\C$ is admissble if and only if for all $b \in V$, $C \in \C$, and $(i,j), (i',j') \in C$, $\dist(i,b) = c(i,j) + \dist(i,b) \Leftrightarrow \dist(i',b) = c(i',j') + \dist(j', b)$.  In other words, for any arbitrary target, all of the edges with the same color are on a shortest path, or none of the edges are.
\end{definition}

We now consider treating all of the edges of the same color as a single entity for regret.  This gives us the notion of local colored regret.
\begin{align}
R^T_{\mathrm{localcolor}} &= \sum_{C \in \C} \left(\sum_{(i,j) \in C} R^T_{i,j}\right)^+
\end{align}

\begin{theorem}\raggedright
If $\C$ is admissible then $R^T_\localexternal \le R^T_\localcolor/D$.
\end{theorem}

\ifincludeinlineproofs
\begin{proof}
\begin{align}
R^T_{\localexternal} &= \max_{b \in A} \left(\sum_{(i,j) \in \Ecloserto{b}} R^T_{i,j}/D\right)^+ \\
&= \max_{b \in A} \left(\sum_{C\in \C} \sum_{(i,j) \in C\cap \Ecloserto{b}} R^T_{i,j}/D\right)^+
\end{align}
For a particular target $b$ let $\C_b = \{C \in \C : C\subseteq \Ecloserto{b}\}$, \ie, $\C_b$ is the set of colors that reduces the distance to $b$.  Then by $\C$'s admissibility,
\begin{align}
R^T_{\localexternal} &= \max_{b \in A} \left(\sum_{C\in \C_b} \sum_{(i,j) \in C}R^T_{i,j}/D\right)^+ \\
&\le \max_{b \in A} \sum_{C\in \C_b} \left(\sum_{(i,j) \in C} R^T_{i,j}/D\right)^+ \\
&\le \sum_{C\in \C} \left(\sum_{(i,j) \in C} R^T_{i
,j}/D\right)^+ \\
&= R^T_{\localcolor} / D 
\end{align}
\end{proof}
\fi

\noindent
So by minimizing local colored regret, we minimize local external regret.  The natural extension of our local swap regret algorithm from the previous section results in an algorithm that can minimize local colored regret.

\begin{definition}
$(b, L, \C)$-colored-regret-matching is the algorithm that initializes $\tilde{R}^0_C = 0$, for all $C\in \C$, 
chooses actions at time $t$ according to a distribution $\pi^t$ that satisfies Requirement~\ref{req:stationary-distribution} with $\tilde{R}^t_{i,j} \equiv \tilde{R}^t_{c(i,j)}$,
and after choosing action $i$ and observing $u^t$ at time $t$ for all $C\in \C$ updates $\tilde{R}^{t}_C = \tilde{R}^{t-1}_C +\sum_{j:(i,j)\in C} (u^t(j) - u^t(i) - b)$.
\end{definition}

\begin{restatable}{theorem}{thmcolorregretbound}
For an arbitrary graph $G$ with maximum degree $D$, arbitrarily chosen vertex $\root$, and edge coloring $\C$, $(\Delta/ (L+1), L, \C)$-colored-regret matching applied after $T$ steps will have expected local colored regret no worse than,
\[
\frac{1}{T} E[R^T_{\mathrm{localcolor}}] \le 
\frac{\Delta D}{L+1} + \frac{\Delta \sqrt{D |C_L|}}{\sqrt{T}}
\]
where $C_L = \{ C \in \C | \exists (i,j) \in C\text{~s.t.~} \level(i) \le L \}$.
\label{thm:color-regret-bound}
\label{thm:last-body}
\end{restatable}

\ifincludeappendix
The proof is in Appendix~\ref{sec:appendix:local-colored-regret}.
\fi
The consequence of this bound depends upon the number of colors needed for an admissible coloring.  Very small admissible colorings are often possible.  The hypercube graph needs only $2n$ colors to give an admissible coloring, which is exponentially smaller than the total number of edges, $n2^n$.  We can also find a reasonably tight coloring for our decision tree graph example, despite being a complex asymmetric graph.

\begin{example}[Colored Decision Tree Learning]\rm
Reconsider Example~\ref{ex:dtree}\ifspacetight. \else~and the graph in Figure~\ref{fig:graphs:dtree}. \fi  Recall that an edge exists between one decision tree and another if the latter can be constructed from the former by replacing a subtree at any node (internal or leaf) with a label (edge length 1) or a stump (edge length 1.1).  We will color this edge with the pair: (i) the sequence of variable assignments that is required to reach the node being replaced, and (ii) the stump or label that replaces it.  This coloring is admissible.  We can see this fact by considering a color: the sequence of variable assignments and resulting stump or label.  If this color is consistent with the target decision tree (i.e., the sequence exists in the target decision tree, and the variable of the added stump matches the variable split on at that point in the target decision tree) then the color must move you closer to the target tree.  
\ifincludeappendix
A formal proof of its admissibility is very involved and can be found in Appendix~\ref{sec:appendix:dtree}.
\fi
\label{ex:dtree-coloring}
\end{example}

\section{Experimental Results}

The previous section presented algorithms that minimize local swap and local external regret (by minimizing local colored regret).  The regret bounds have no dependence on the size of the graph beyond the graph's degree, and so provide a guarantee even for infinite graphs.  
We now explore these algorithms' practicality as well as illustrate the generality of the concepts by applying them to a diverse set of online problems.  The first two tasks we examine, online Max-3SAT and online decision tree learning, have not previously been explored in the online setting.  The final task, online disjunct learning, has been explored previously, and will help illustrate some drawbacks of local regret.  

In all three domains we examine two algorithms.  The first minimizes local swap regret by applying $(\Delta/(L+1), L)$-regret matching with $L$ chosen specifically for the problem.  This will be labeled ``Local Swap''.  The second focuses on local external regret by using a tight, admissible edge-coloring and applying $(\Delta/(L+1), L, \C)$-colored-regret matching.  This will be labeled simply ``Local External''.

\subsection{Online Max-3SAT}
First, we consider Example~\ref{ex:max-cnf}.  We randomly constructed problem instances with $n=20$ boolean variables and 201 clauses each with 3 literals.  
On each timestep, the algorithms selected an assignment of the variables, a clause was chosen at random from the set, and the algorithm received a utility of 1 if the assignment satisfied the clause, 0 otherwise.  This was repeated for 1000 timesteps.  
The locality graph used was the $n$-dimensional hypercube from Example~\ref{ex:max-cnf}.  The admissible coloring used to minimize local external regret was the $2n$ coloring that has two colors per variable (one for turning the variable on, and one for turning the variable off).  In both cases we set $L = \infty$ and $b = 0$, since the bounds do not depend on $L$ once it exceeds 20.  This also achieved the best performance for both algorithms.  The average results over 200 randomly constructed sets of clauses are shown in Figure~\ref{fig:results-3cnf}, with 95\% confidence bars.

\iftwocolumns
\begin{figure}
\centering\begin{tabular}{cc}
(a) & \basebox{\includegraphics[width=0.75\hsize]{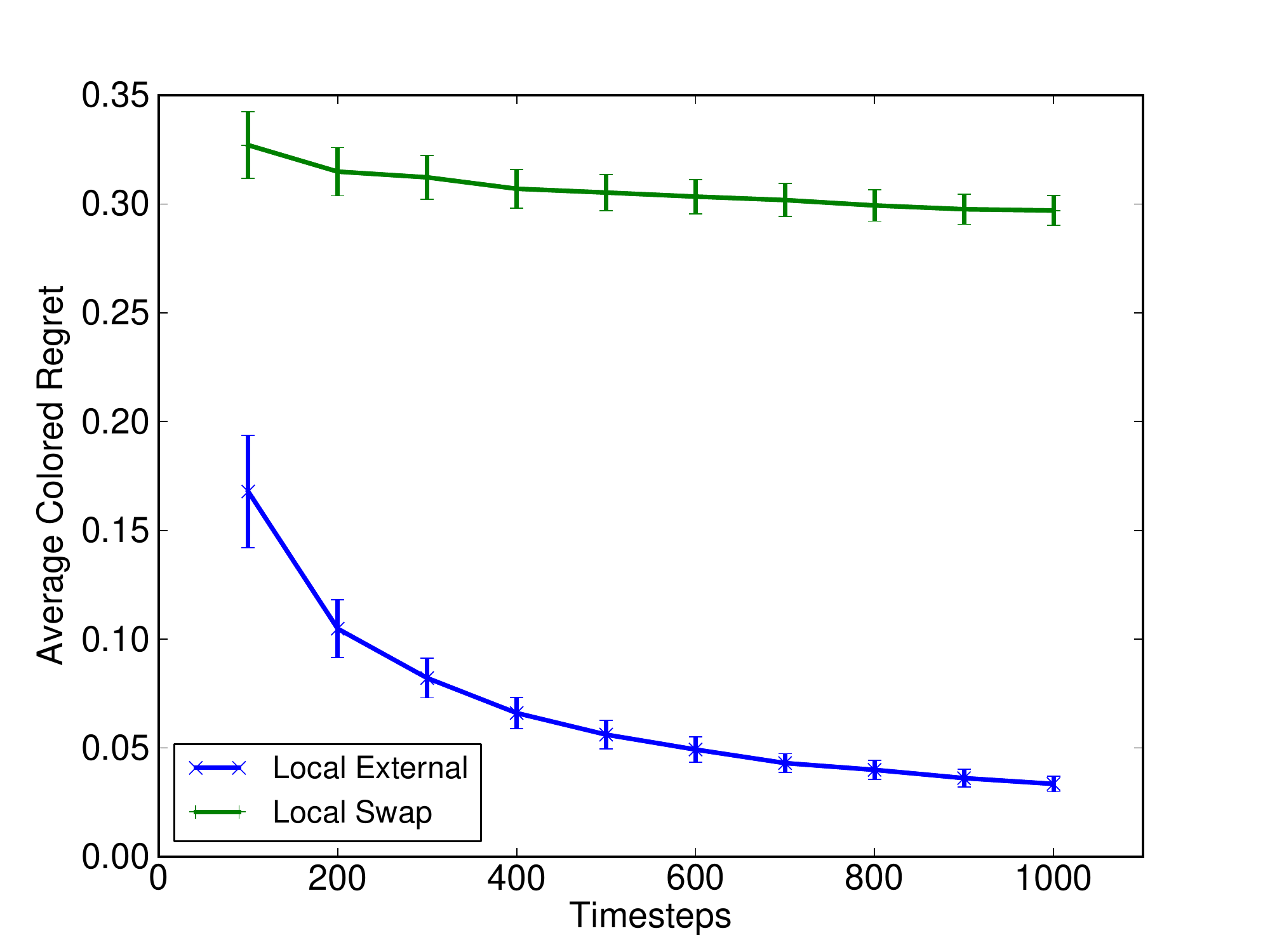}} \\
(b) & \basebox{\includegraphics[width=0.75\hsize]{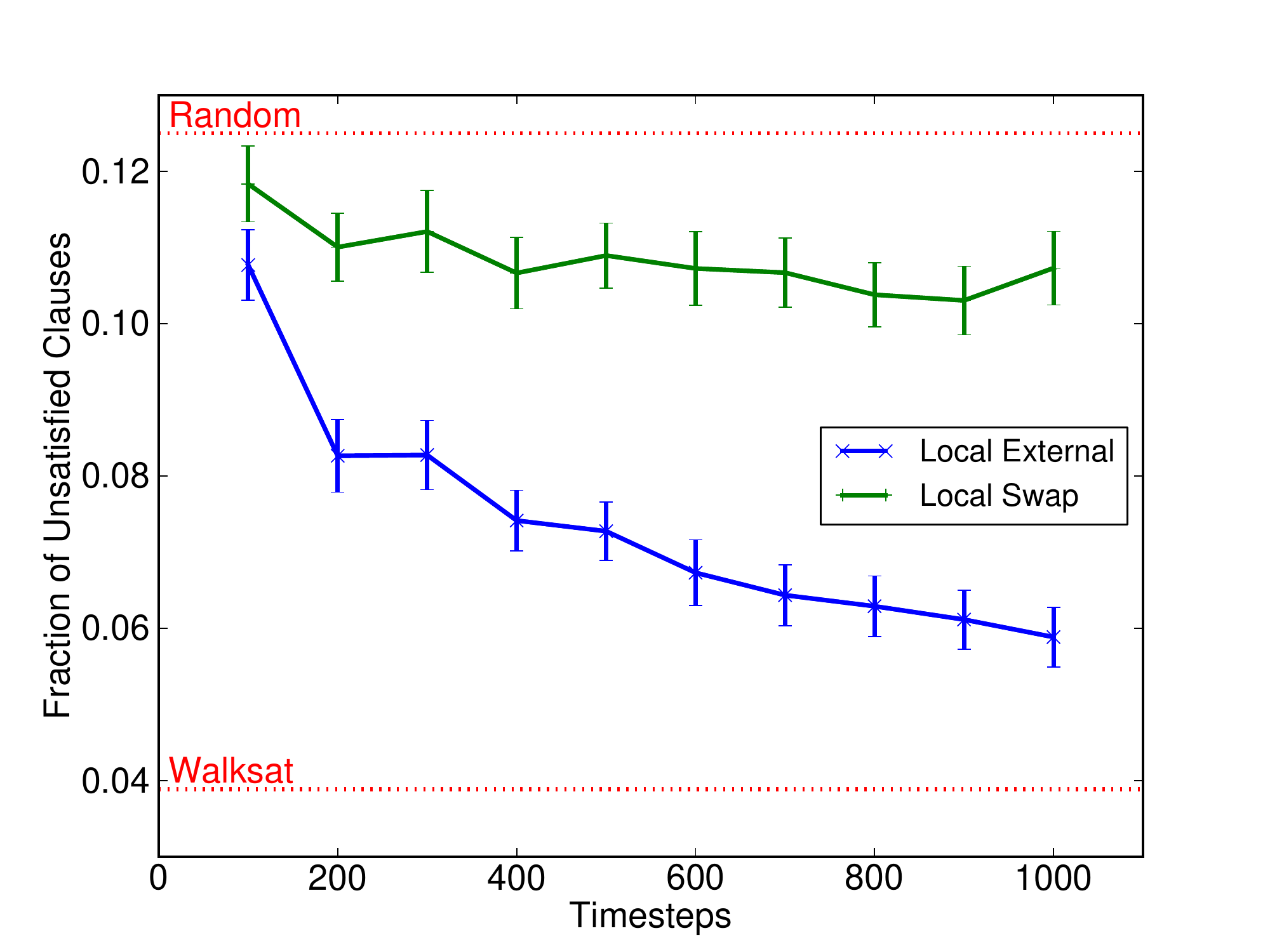}}
\end{tabular}
\caption{Results for Online Max-3SAT: (a) regret, (b) fraction of unsatisfied clauses.}
\label{fig:results-3cnf}
\end{figure}
\else
\begin{figure}
\centering\begin{tabular}{cc}
\basebox{\includegraphics[width=0.47\hsize]{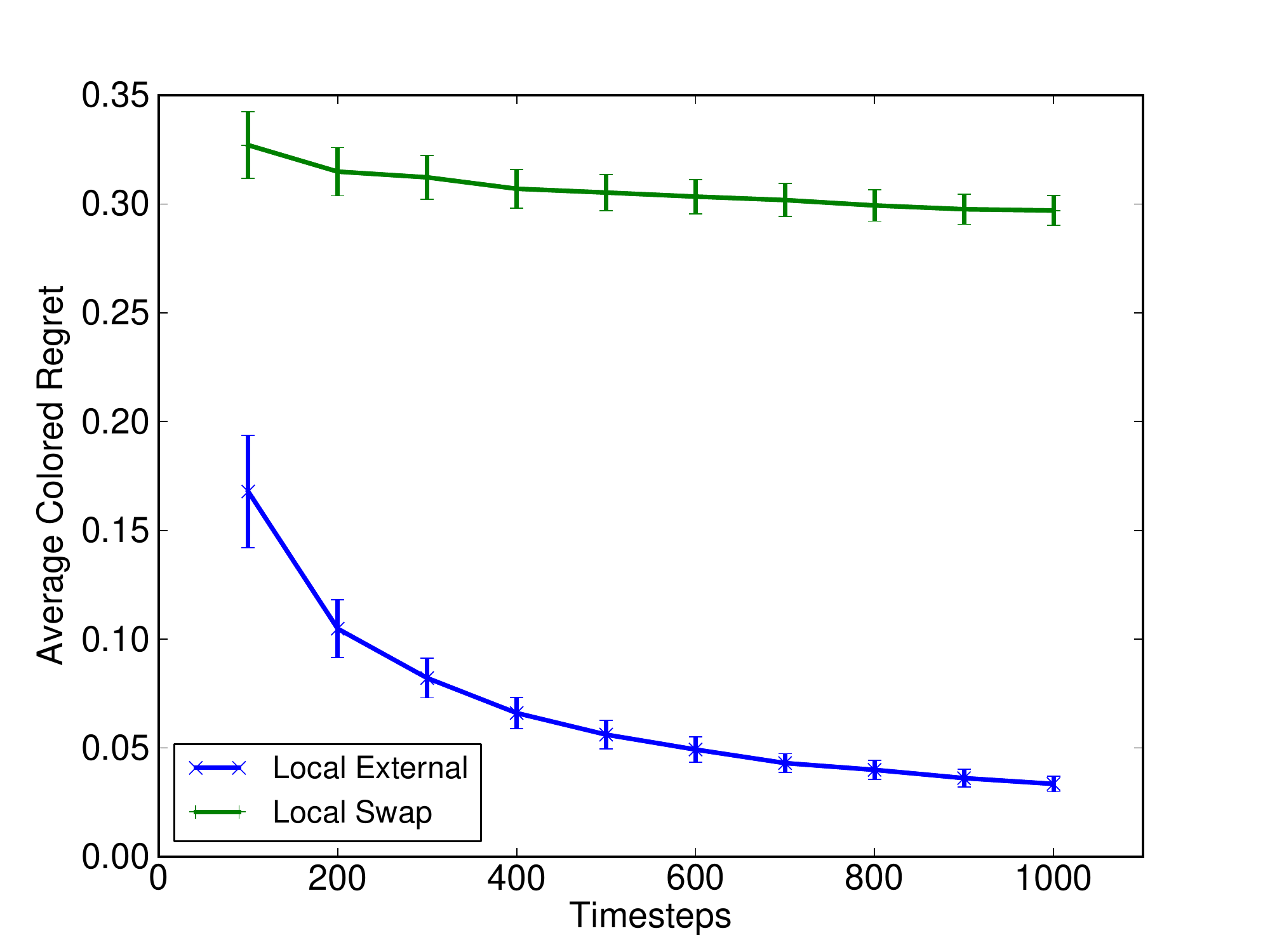}} &
\basebox{\includegraphics[width=0.47\hsize]{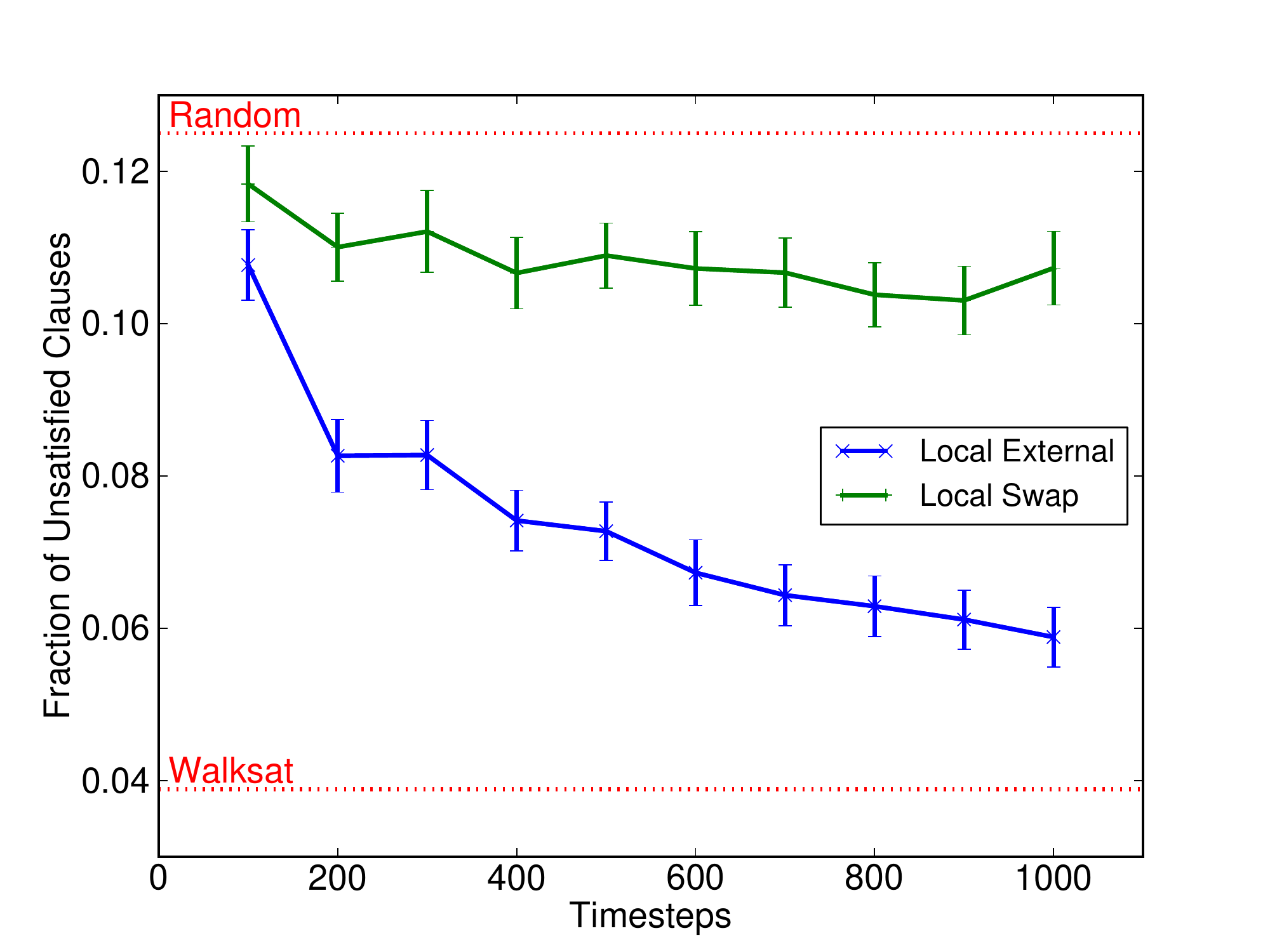}} \\
(a) & (b) 
\end{tabular}
\caption{Results for Online Max-3SAT: (a) regret, (b) fraction of unsatisfied clauses.}
\label{fig:results-3cnf}
\end{figure}
\fi

Figure~\ref{fig:results-3cnf} (a) shows the time-averaged colored regret of the two algorithms, to demonstrate how well the algorithms are actually minimizing regret.  Both are decreasing over time, while external regret is decreasing much more rapidly.  As expected, swap regret may be a stronger concept, but it is more difficult to minimize.  The local external regret algorithm after only one time step can have regret for not having made a particular variable assignment, while local swap regret has to observe regret for this assignment from every possible assignment of the other variables to achieve the same result.  This is further demonstrated by the number of regret values each algorithm is tracking: local external regret on average had ~34 non-zero regret values, while local swap regret had ~4200 non-zero regret values. In summary, external regret provides a powerful form of generalization.  Figure~\ref{fig:results-3cnf} (b) shows the fraction of the previous 100 clauses that were satisfied.  Two baselines are also presented.  A random choice of variable assignments can satisfy $\frac{7}{8}$ of the clauses in expectation.  We also ran WalkSAT~\citep{SeKaCo93} offline on the set of 201 clauses, and on average it was able to satisfy all but ~4\% of the clauses, which gives an offline lower bound for what is possible.  Both substantially outperformed random, with the external regret algorithm nearing the performance of the offline WalkSat.

\subsection{Online Decision Tree Learning}
Second, we consider Example~\ref{ex:dtree}.  We took three datasets from the UCI Machine Learning Repository~(each with categorical inputs and a large number of instances): nursery, mushroom, and king-rook versus king-pawn~\citep{UCI}.  The categorical attributes were transformed into boolean attributes (which simplified the implementation of the locality graphs) by having a separate boolean feature for each attribute value.\footnote{As a result, there were $n=28$ features for nursery, $118$ features for mushroom, and $74$ features for king-rook versus king-pawn.}  We made the problems online classification tasks by sampling five instances at random (with replacement) for each timestep, with the utility being the number classified correctly by the algorithm's chosen decision tree.   This was repeated for 1000 timesteps, and so the algorithms classified 5,000 instances in total.  The locality graph used was the one described in Example~\ref{ex:dtree}.  The tight coloring used to minimize local external regret was the one described in Example~\ref{ex:dtree-coloring}.  $L$ was set to 3 for local swap regret, and 100 for local external regret, as this achieved the best performance.  Even with the far larger graph, the external regret algorithm was observing nearly one-eighth of the number of non-zero regret values observed by the local swap algorithm.
The average results over 50 trials are shown in Figure~\ref{fig:results-dtree}(a)-(c) with 95\% confidence bars.

\iftwocolumns
\begin{figure}
\centering
\begin{tabular}{cc}
(a) & \basebox{\includegraphics[width=0.75\hsize]{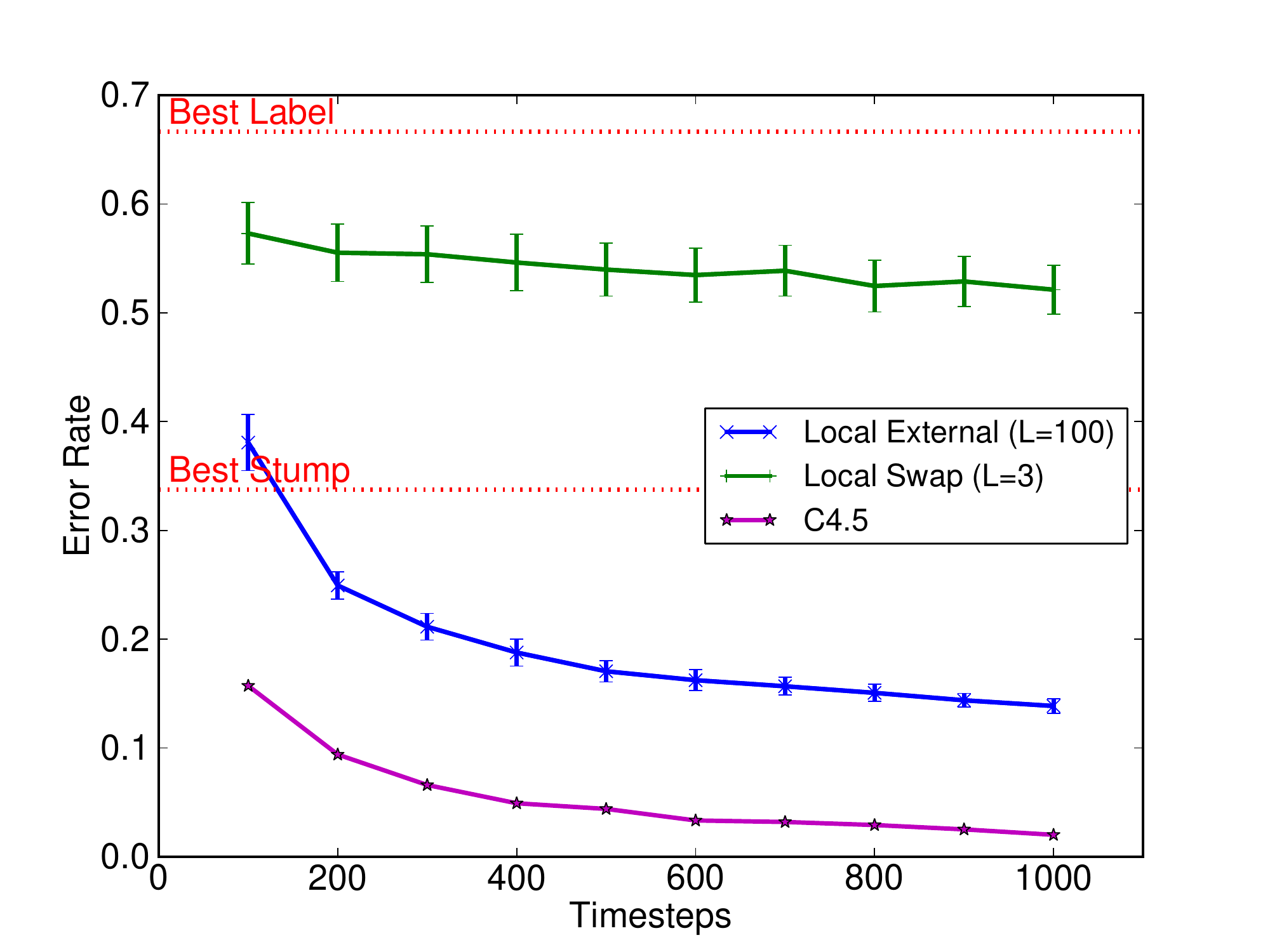}} \\
(b) & \basebox{\includegraphics[width=0.75\hsize]{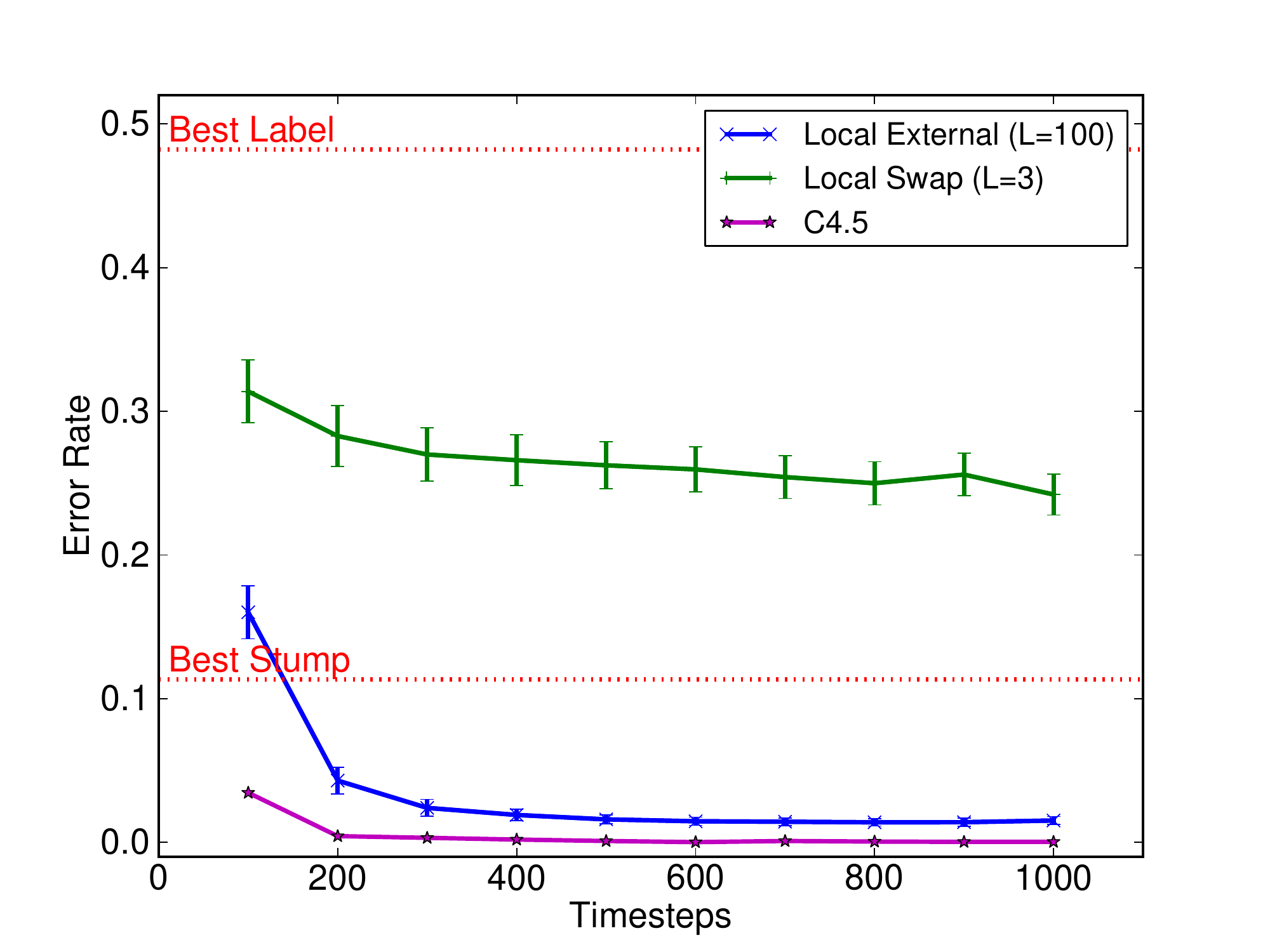}} \\
(c) & \basebox{\includegraphics[width=0.75\hsize]{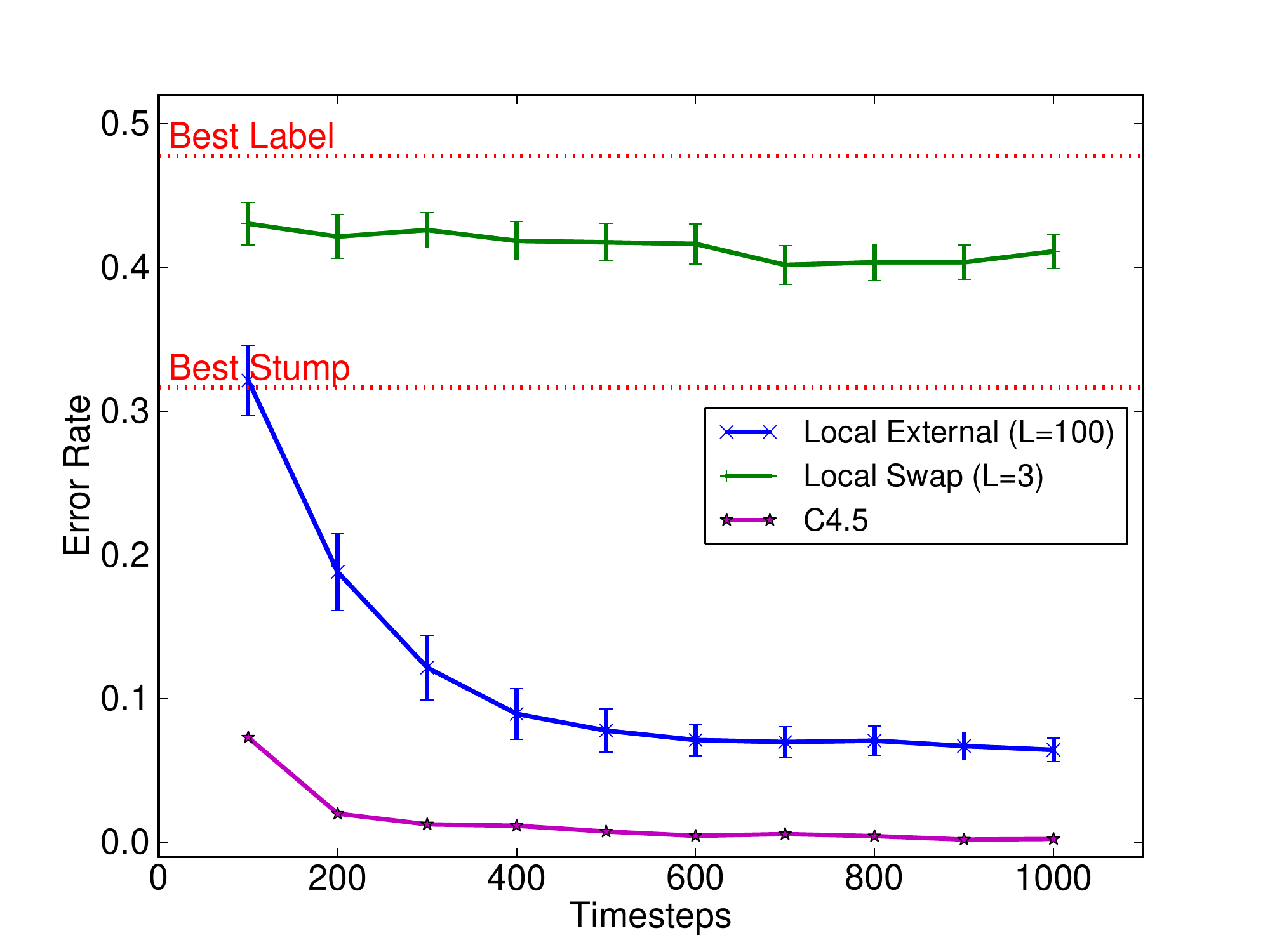}} \\
(d) & \basebox{\includegraphics[width=0.75\hsize]{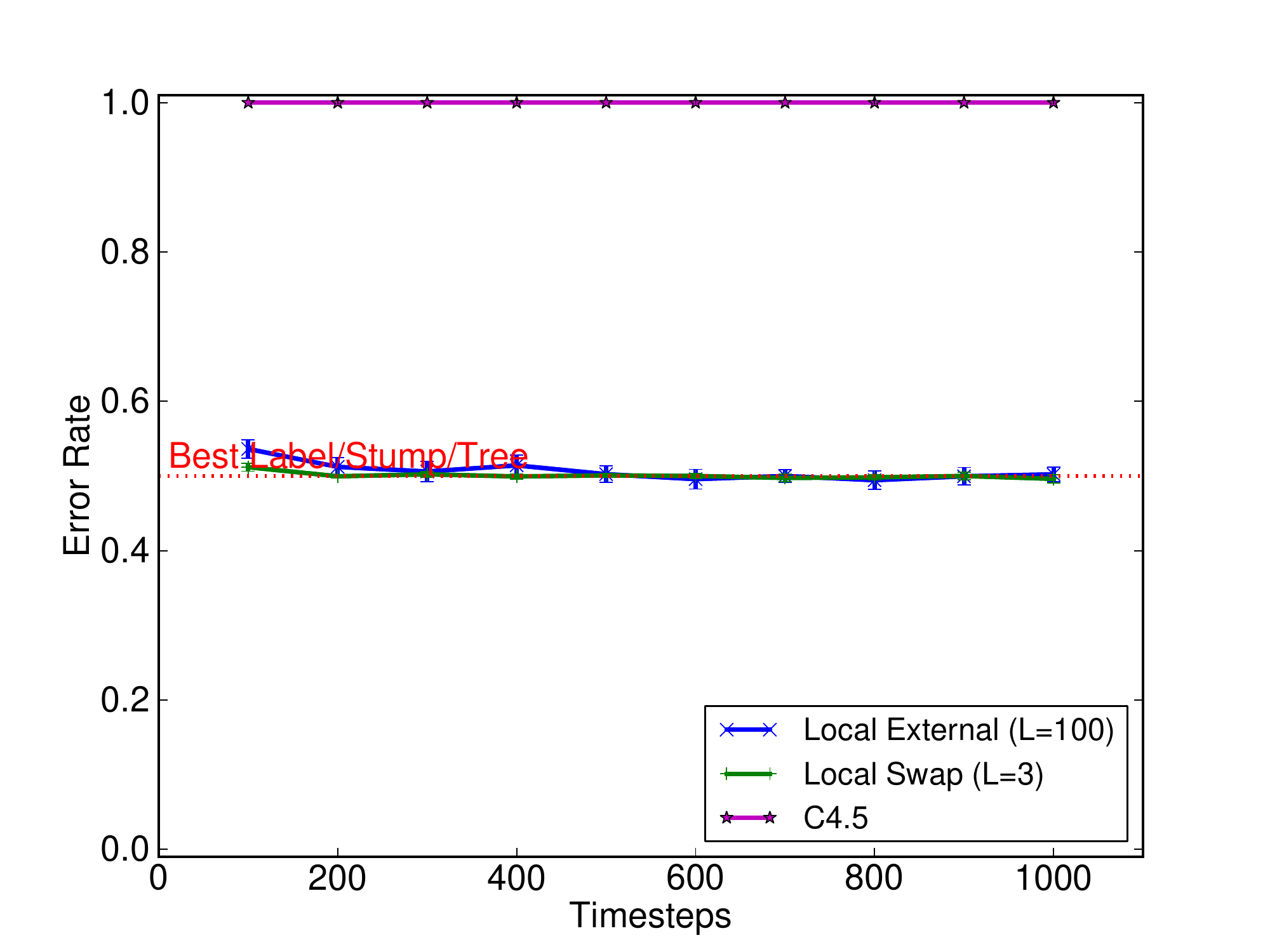}}
\end{tabular}
\caption{Results for online decision tree learning on three UCI datasets: (a) Nursery, (b) Mushroom, (c) King-Rook/King-Pawn; and (d) a simple sequence of alternating labels.}
\label{fig:results-dtree}
\end{figure}
\else
\begin{figure}
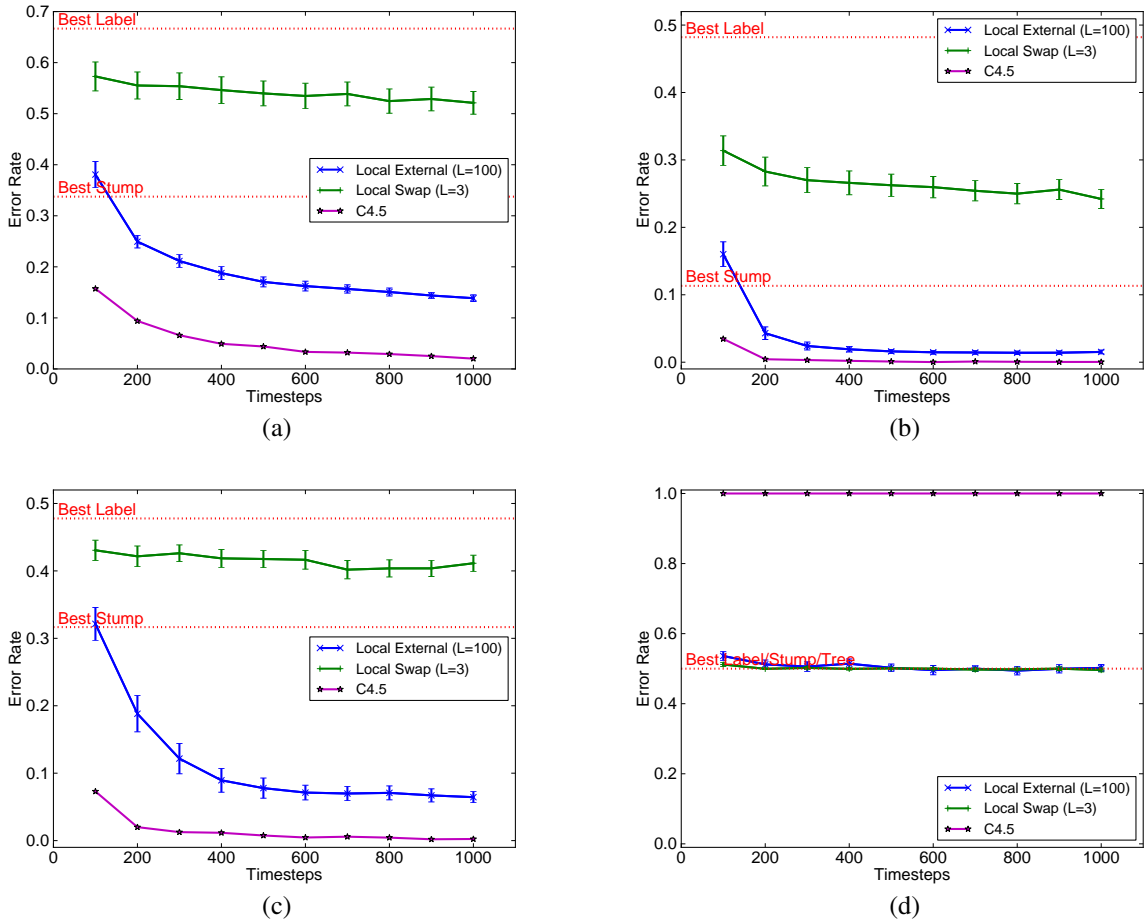

\centering
\begin{tabular}{cc}
\basebox{\includegraphics[width=0.48\hsize]{figs/nursery-results}} &
\basebox{\includegraphics[width=0.48\hsize]{figs/mushroom-results}} \\
(a) & (b) \\
\basebox{\includegraphics[width=0.48\hsize]{figs/kr-vs-kp-results}} &
\basebox{\includegraphics[width=0.48\hsize]{figs/tricky-results}} \\
(c) & (d) 
\end{tabular}
\caption{Results for online decision tree learning on three UCI datasets: (a) Nursery, (b) Mushroom, (c) King-Rook/King-Pawn; and (d) a simple sequence of alternating labels.}
\label{fig:results-dtree}
\end{figure}
\fi

The graphs show the average fraction of misclassified instances over the previous 100 timesteps.  Two baselines are also plotted: the best single label (i.e., the size of the majority class) and the best decision stump.  Both regret algorithms substantially improved on the best label, and local external regret was selecting trees substantially better than the best stump.  As a further baseline, we ran the batch algorithm C4.5 in an online fashion, by retraining a decision tree after each timestep using all previously observed examples.  C4.5's performance was impressive, learning highly accurate trees after observing only a small fraction of the data.  However, C4.5 has no regret guarantees.  As with any offline algorithm used in an online fashion, there is an implicit assumption that the past and future data instances are i.i.d..  In our experimental setup, the instances were i.i.d., and as a result C4.5 performed very well.  To further illustrate this point, we constructed a simple online classification task where instances with identical attributes were provided with alternating labels.  The best label (as well as the single best decision tree) has a 50\% accuracy.  C4.5 when trained on the previously observed instances, misclassifies every single instance.  This is shown along with local regret algorithms in Figure~\ref{fig:results-dtree} (d).

\subsection{Online Disjunct Learning}
%
%
Finally, we examine online disjunct learning as described in Example~\ref{ex:disjuncts}.  This task has received considerable attention, notably the celebrated Winnow algorithm~\citep{Little88}, which is guaranteed to make a finite number of mistakes if the instances can be perfectly classified by some disjunction.  Furthermore, the number of mistakes Winnow2 makes, when no disjunction captures the instances, can be bounded by the number of attribute errors (\ie, the number of input attributes that must be flipped to make the disjunction satisfy the instance) made by the best disjunction.  In these experiments we compare our algorithms' performance to that of Winnow2.  

We looked at two learning tasks.  In the first, we generated a random disjunction over $n=20$ boolean variables, where a variable was independently included in the disjunction with probability $4/n$.  Instances were created with uniform random assignments to all of the variables, with a label being true if and only if the chosen disjunct is true for the instance's assignment.  In the second case, we chose instances uniformly at random from a constructed set of 21 instances: one for each variable with that variable (only) set to true and the label being true, and one with all of the variables assigned the value of true and the label being false.  We call this task Winnow Killer.  For both tasks, the $n$-dimensional hypercube from Example~\ref{ex:max-cnf} was used as the locality graph with the $2n$ coloring as our admissible coloring, and $L=\infty$ and $b=0$.  The average results over 50 trials are shown in Figure~\ref{fig:results-disjunct}, with 95\% confience bars.

\iftwocolumns
\begin{figure}
\centering\begin{tabular}{cc}
(a) & \basebox{\includegraphics[width=0.75\hsize]{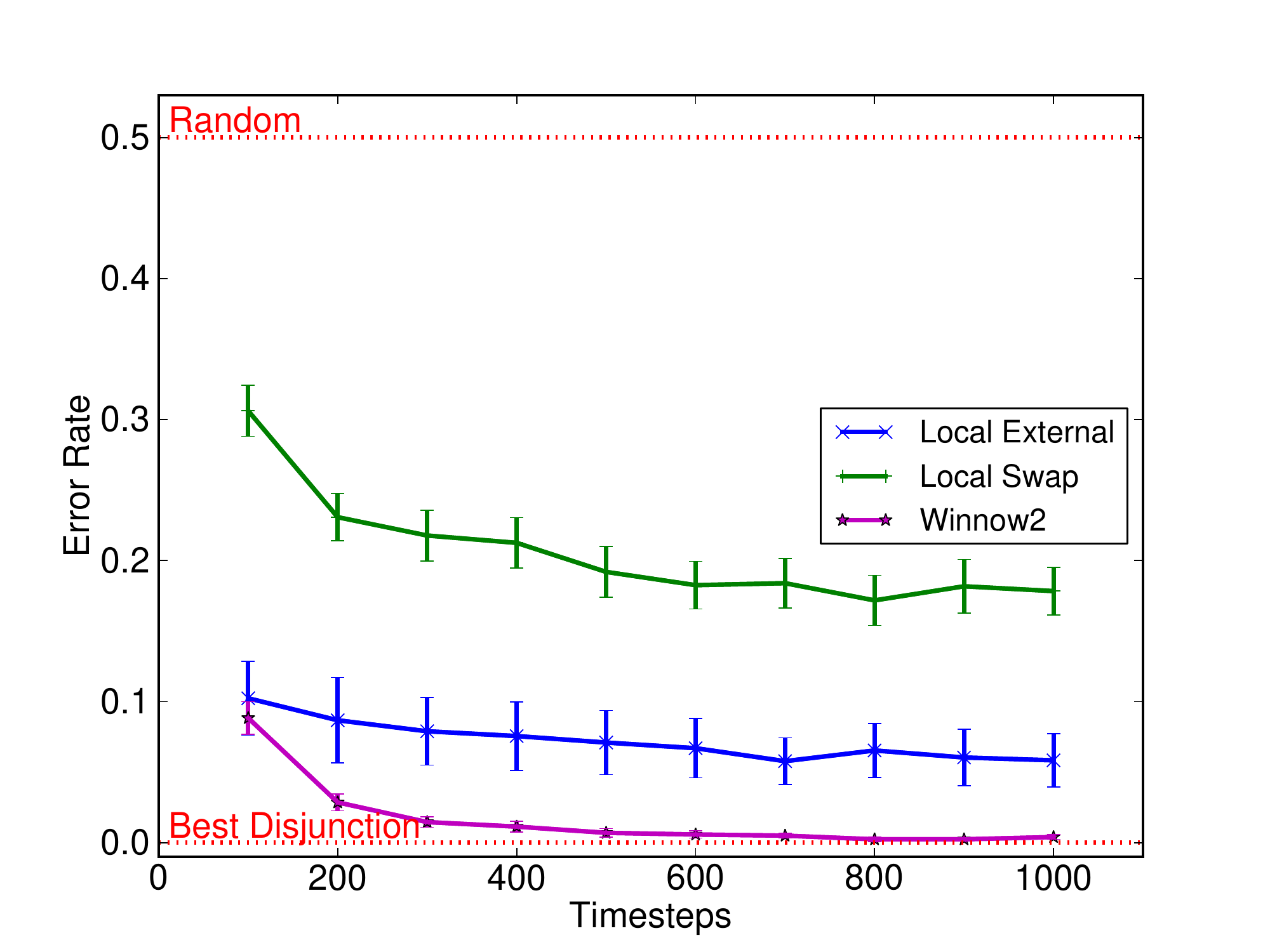}} \\
(b) & \basebox{\includegraphics[width=0.75\hsize]{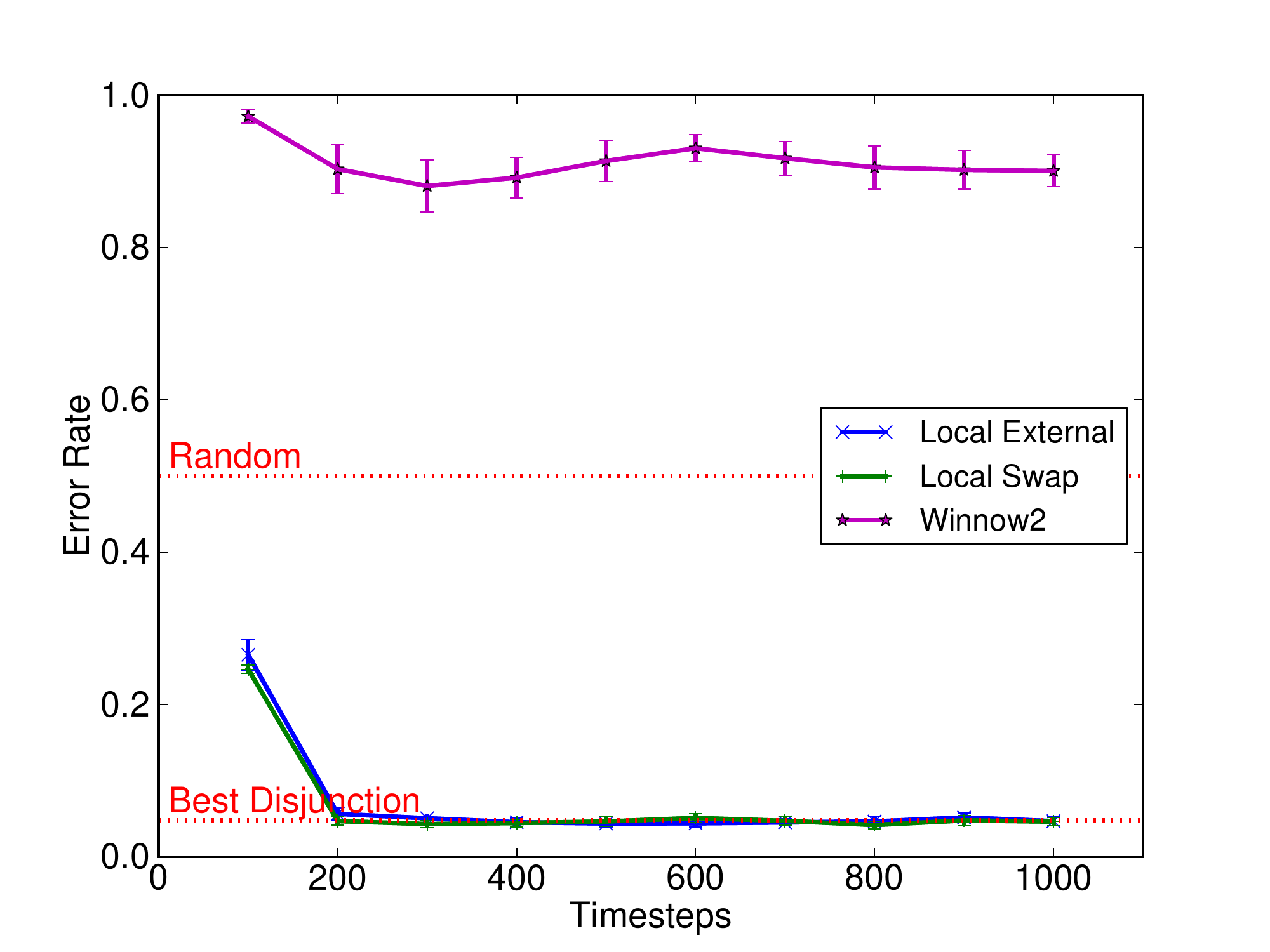}}
\end{tabular}
\caption{Results for online disjunct learning: (a) random disjunct, (b) Winnow Killer.}
\label{fig:results-disjunct}
\end{figure}
\else
\begin{figure}
\centering\begin{tabular}{cc}
\basebox{\includegraphics[width=0.48\hsize]{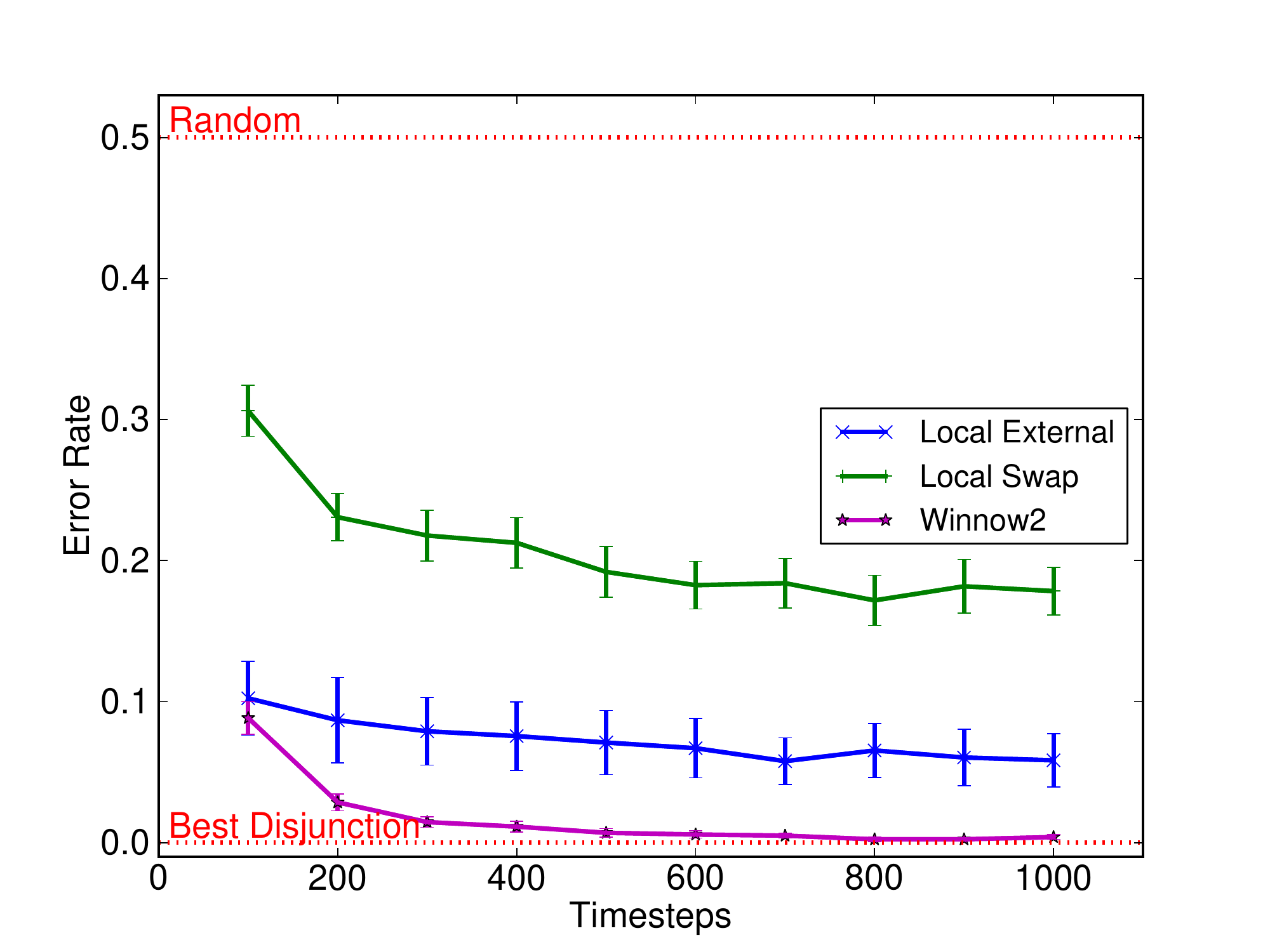}} &
\basebox{\includegraphics[width=0.48\hsize]{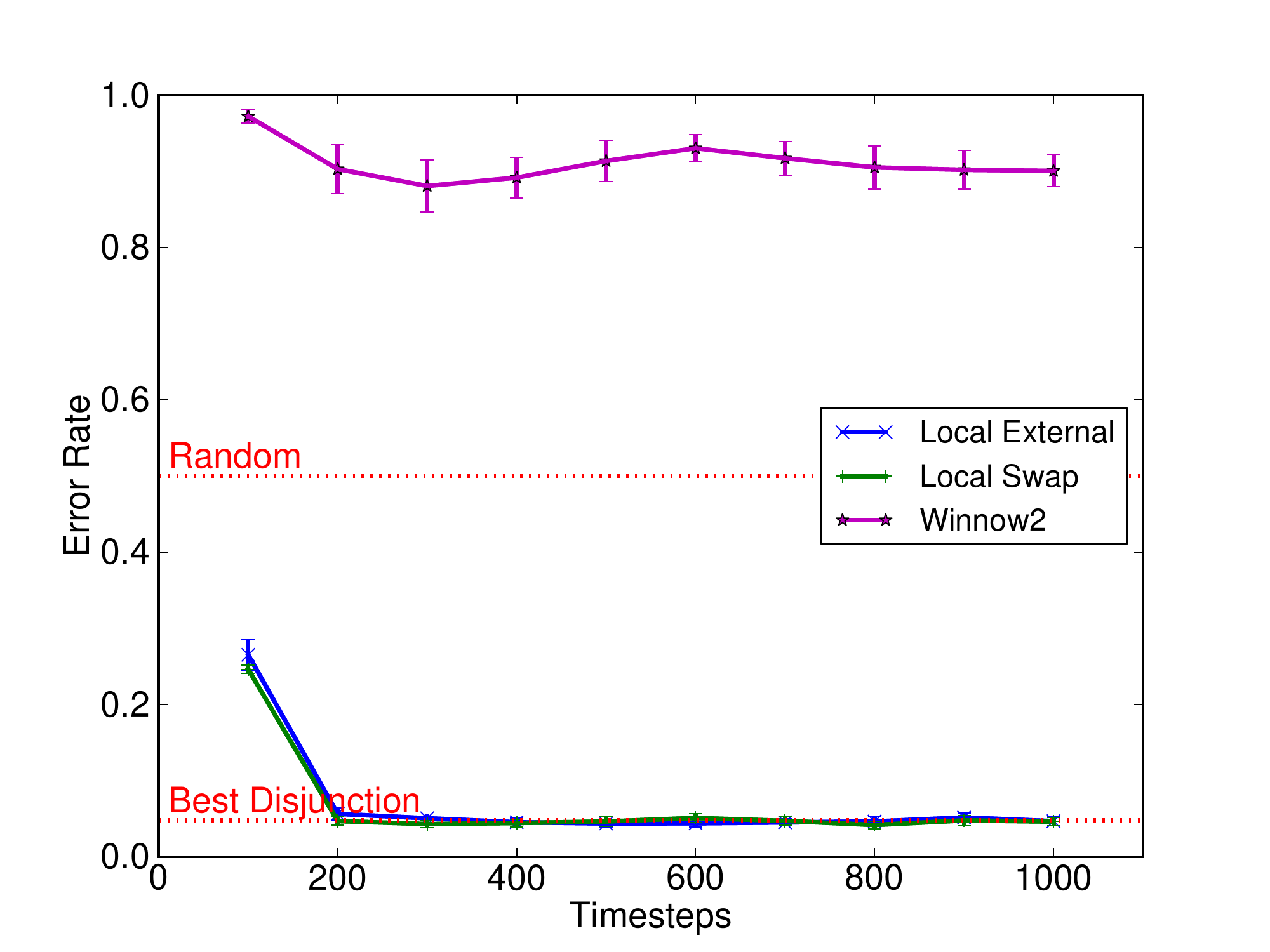}} \\
(a) & (b)
\end{tabular}
\caption{Results for online disjunct learning: (a) random disjunct, (b) Winnow Killer.}
\label{fig:results-disjunct}
\end{figure}
\fi

The graphs plot error rates over the previous 100 instances.  Three baselines are plotted: randomly assigning a label (guaranteed to get half of the instances correct on expectation), the best disjunct (which makes no mistakes for random disjunctions and makes $\frac{1}{21}$ mistakes on the Winnow Killer task), and Winnow2.  Figure~\ref{fig:results-disjunct} (a) shows the results on random disjunctions.  Winnow2 is guaranteed to make a finite number of mistakes and indeed its error rate drops to zero quickly.  The local regret concepts, though, have difficulties with random disjunctions.  The reason can be easily seen for the case of local external regret.  Suppose the first instance is labeled true; the algorithm now has regret for all of the variables that were true in that instance (some of these will be in the target disjunction, but many will not).  These variables will now be included in the chosen disjunction for a very long time, as the only regret that one can have for not removing them is if their assignment was the sole reason for misclassifying a false instance.  In other words, the problem is that there's no regret for not removing multiple variables simultaneously as this is not a local change.  Winnow2, though, also has issues.  It performs very poorly in the Winnow Killer task (in fact, if the instances were ordered it could be made to get every instance wrong), as shown in Figure~\ref{fig:results-disjunct} (b).  Since the mistake bound for Winnow2 is with respect to the number of attribute errors, a single mistake by the best disjunction can result in $n$ mistakes by Winnow2.  A further issue with Winnow is that while its peformance is tied to the performance of disjunctions, its own hypothesis class is not disjunctions but a thresholded linear function, whereas local regret is playing in the same class of hypotheses that it comparing against.

\section{Conclusion}

We introduced a new family of regret concepts based on restricting regret to only nearby hypotheses using a locality graph.  We then presented algorithms for minimizing these concepts, even when the number of hypotheses are infinite.  Further we showed that we can exploit structure in the graph to achieve tighter bounds and better performance.  These new regret concepts mimic local search methods, which are common approaches to offline optimization with intractably hard hypothesis spaces.  As such, our concepts and algorithms allows us to make online guarantees, with a similar flavor to their offline counterparts, with these hypothesis spaces.

\ifspacetight\else
There is a number of interesting directions for future work as well as open problems.  Admissible colorings can result in radically improved bounds as well as empirical performance.  How can such admissible colorings be constructed for general graphs?  What graph structures lead to exponentially small admissible colorings compared to the size of the graph?  We can easily construct the minimum admissible coloring for graphs that are recursively constructed as Cartesian product of graphs and complete graphs.  While such graphs can have exponentially small admissible colorings, they form a very narrow class of structures.  What other structures lead to exponentially small admissible colorings?  Furthermore, edge lengths can have a significant impact on the size of the minimum admissible coloring.  For example, the decision tree graph from Example~\ref{ex:dtree} was carefully constructed to result in a tight coloring, and, in fact, unit length edges over the same graph would result in an exponentially larger admissible coloring.  How can edge lengths be defined to allow for small minimum colorings?
\fi

\section*{Acknowledgements}

This work was supported by NSERC and Yahoo! Research, where the first author was a visiting scientist at the time the research was conducted. 

\ificmlformat
{\small
\bibliographystyle{icml2012}
\bibliography{paper}
}
\else
\bibliographystyle{named}
\bibliography{paper}
\fi
\fi

\ifincludeappendix
\ifincludebody\newpage\fi

\appendix
\section{Proof for Local Swap Regret}

\label{sec:appendix:local-swap-regret} 
At its heart, the Hart and Mas-Colell proof for minimizing internal
regret relies on the relationship between Markov chains and flows. The
Blackwell condition is (roughly speaking) that the probability flow into an action equals
the probability flow out of an action. In the variant here, there are
two ways to view this flow.  Define $f$ such that for all $(i,j)\in E$, $f_{i,j}=\pi^{t+1}_{i}
\tilde{R}^{t,+}_{i,j}$. Implicitly, $f$ depends on the time $t$, but
we supress this as we always refer to a time $t$.
This flow $f$ is similar to the flows in Hart and Mas-Colell as they apply to the
Blackwell condition. However, it
lacks the conservation of flow property. Thus, we consider a second
flow $f'$ which satisfies the conservation of flow. To do this, we consider the levels of the graph. To review, $\Root$ is a distinct vertex, and,
$\Level(v)=d_1(\Root,v)$. If we consider the flow $f$ as starting
from the root, it (roughly) goes from level to level outward from the root until
it reaches level $L$. Then, while $f$ flows to level $L+1$ and reaches
a dead end (violating the conservation property), $f'$ is switched,
and flows back to the $\Root$. In order to make
the proof work, we have to bound the difference between $f$ and
$f'$. Since this difference is mostly on the flow from level $L$ to
level $L+1$, we need to bound the fraction of the total flow that is
going out of the last level by showing that this flow is less than
the flow going from the root to the first level, and it is less than
the flow from the first level
to the second level, et cetera.

First we show that for nodes on most levels, the flow in
equals the flow out. 

\begin{lemma}
If Requirement~\ref{req:stationary-distribution} holds, then for all $j \in V$ such that $1 \le \Level(j) \le L$,
\[
\sum_{i:(i,j)\in E} f_{i,j}= \sum_{k:(j,k)\in E} f_{j,k}  
\]
\label{lem:incoming-outgoing}
\end{lemma}

\begin{corollary}\label{cor:incoming-outgoing}
By summing over the nodes in level $\ell$, for any level $1\leq \ell \leq L$,\[
 \sum_{(i,j)\in E:\Level(j)=\ell} f_{i,j} = \sum_{(i,j)\in E:\Level(i)=\ell} f_{i,j}.
\]
\end{corollary}
\begin{proof}
From Requirement~\ref{req:stationary-distribution}(c) we know there
exists an $M>0$ such that:
\begin{align} 
\pi^{t+1}_j &= \sum_{i:(i,j) \in E} (R^{t,+}_{i,j} / M) \pi^{t+1}_i + \left(1 - \sum_{k:(j,k)\in E} R^{t,+}_{j,k}/M\right) \pi^{t+1}_j \\
\pi^{t+1}_j \left(\sum_{k:(j,k)\in E} R^{t,+}_{j,k} / M\right) &= \sum_{i:(i,j)\in E} R^{t,+}_{i,j} \pi^{t+1}_i / M  \\
\sum_{k:(j,k)\in E} \pi^{t+1}_j R^{t,+}_{j,k} &= \sum_{i:(i,j)\in E}
R^{t,+}_{i,j} \pi^{t+1}_i 
\end{align}
The lemma follows by the definition of $f_{i,j}$.
\end{proof}

If we want the conservation of flow to hold for \textit{all} nodes,
then we need to define a slightly different flow. We want to say that
the flow which is currently exiting the first $L$ levels (specifically
between level $L$ and level $L+1$) is actually
flowing back into the root. So, we want to subtract the edges
$E''=\{(i,j)\in E:\Level(j)\geq L+1\vee \Level(i)\geq L+1\}$, and add the 
edges $E'=\{i\in V:\Level(i)=L\} \times \{\Root\}$. For any edge $e\in
E'$, define $f'_{i,j}=f_{i,j}+\sum_{k:\Level(k)=L+1,(i,k)\in E}
f_{i,k}$., where $f_{i,j}=0$ if $(i,j)\notin E$. For any edge $(i,j)\in
E\backslash (E'\cup E'')$ where $\Level(i),\Level(j)\leq L$,
$f'_{i,j}=f_{i,j}$. Define $\tilde{E}=(E\cup E'\backslash E'')$. 

Thus, we now have a flow over a graph
$(V,\tilde{E})$, but we must prove conservation of flow.

\begin{lemma}\label{lem:modifiedflowconservation}If Requirement~\ref{req:stationary-distribution} holds, for any $i\in V$, $\sum_{j:(i,j)\in \tilde{E}} f'_{i,j}=\sum_{j:(j,i)\in \tilde{E}} f'_{j,i}$
\end{lemma}
\begin{proof}
For $\Level(i)\in \{1\ldots L-1\}$, this is a direct result of Requirement~\ref{req:stationary-distribution}(c).
For when $\Level(i)>L$, there is no flow out or in, making the
result trivial. For $\Level(i)=0$ (when $i=\root$), this is a direct result of
Requirement~\ref{req:stationary-distribution}(d). For when  $\Level(i)=L$, note
that $\sum_{j:(j,i)\in E}f_{j,i}=\sum_{j:(i,j)\in E}f_{i,j}$, for all
$j$ where $(j,i)\in E$, $f_{j,i}=f'_{j,i}$, and for all $(i,j)\in E$
where $\Level(j)\in \{1\ldots L\}$, $f_{i,j}=f'_{i,j}$ and that
the flow $\sum_{j:(i,j)\in E,\Level(j)\in \{0,L+1\}}f_{i,j}=f'_{i,\Root}$, so 
\begin{align}
\sum_{j:(j,i)\in E}f'_{j,i}&=\sum_{j:(j,i)\in E}f_{j,i}\\
&=\sum_{j:(i,j)\in E}f_{i,j}\\
&=\sum_{j:(i,j)\in E,\Level(j)\in \{0,L+1\}}f_{i,j}+\sum_{j:(i,j)\in
  E,\Level(j)\in \{1\ldots L\}}f_{i,j}\\
&=f'_{i,\Root}+\sum_{j:(i,j)\in
  E,\Level(j)\in \{1\ldots L\}}f'_{i,j}\\
&=\sum_{j:(i,j)\in E}f'_{i,j}.
\end{align}
\end{proof}

\begin{lemma}
If Requirement~\ref{req:stationary-distribution} holds, then:
\begin{gather}
\sum_{(i,j)\in E} f_{i,j} \geq (L+1)\sum_{(i,j) \in E:\Level(j)=L+1} f_{i,j}\label{eqn:levelregrettotal}\\
\sum_{(i,j)\in E:L(i)=0} f_{i,j} \geq \sum_{(i,j) \in E:\Level(j)=L+1} f_{i,j}\label{eqn:levelregrettopbottom}
\end{gather}
\label{lem:levelregret-dag}
\end{lemma}

\begin{proof}
To obtain an intuition, consider the case where all outgoing edges
from level $j$ go to level $j+1$ (modulo the last level). In this
case, the flow from level 0 all goes to level 1, from there goes to
level 2, and so forth until it reaches level $L$ and then returns to
level $0$. Thus, the inflows and outflows of all the levels would be
equal.
The problem with this is that outgoing edges from level $j$ can go to
other nodes in $j$, or nodes in level $j-1$, et cetera. At an
intuitive level, a backwards flow would not make more flow through the
final level, any more than an eddy would somehow create water at the
mouth of a river, and we must simply formally prove this.

First, we define $g_{i,j}=\sum_{(k,l)\in
  E:\Level(k)=i,\Level(l)=j}f'_{i,j}$, the total flow between
levels. By Lemma~\ref{lem:modifiedflowconservation} for all $i\in V$,
$\sum_{j}f'_{i,j}=\sum_{j}f'_{j,i}$, so the aggregate flow satisfies the conservation of
flow, namely that for all $i$, $\sum^L_{j=0}g_{i,j}=\sum^L_{j=0} g_{j,i}$. Also,
if $j>i+1$, then $g_{i,j}=0$.
Define $n_i=g_{i,i+1}$, the flow between one level and the
next. Since $f$, $f'$, and $g$ are just different groupings of the
total flow throughout the graph, $\sum_{(i,j)\in E} f_{i,j}=\sum_{(i,j)\in \tilde{E}} f'_{i,j}=
\sum_{i=0}^L\sum_{j=0}^L g_{i,j}$. Since for all $i,j\in V$, $f'_{i,j}\geq
0$, then for all $i,j$, $g_{i,j}\geq 0$.
 $g_{L,0}+\sum_{i=0}^{L-1} n_i\leq \sum_{(i,j)\in E} f_{i,j}$.

Moreover, $n_0=g_{0,1}=\sum_{j:(\Root,j)\in
  E}f'_{\Root,j}=\sum_{j:(\Root,j)\in E}f_{\Root,j}$, and
$g_{L,0}\geq \sum_{(i,j)\in E:\Level(j)=L+1}f_{i,j}$. So if we prove that
for all $i$, $g_{L,0}\leq n_i$, then $g_{L,0}\leq n_0$ and
that $g_{L,0}(L+1)\leq g_{L,0}+\sum_{i=0}^{L-1} n_i$, we have proven
the lemma.

First, we identify this backwards flow. Define $\delta_i$ to be the
flow that originates at level $i$ or above and flows back to a lower
level. Formally, define $\delta_0=0$, and $\delta_i=\sum_{i'<i,j'\geq i} g_{j',i'}-g_{L,0}$. Note that $\delta_i\geq 0$.

Thus, for all $i$ where $0<i<L$:

\begin{align}
\delta_i-\delta_{i+1}&=\left (\sum_{i'<i,j'\geq i} g_{j',i'}\right )-\left (\sum_{i'<i+1,j'\geq i+1} g_{j',i'}\right )\\
\delta_i-\delta_{i+1}&=\left (\sum_{i'<i,j'= i} g_{j',i'}\right )+\left (\sum_{i'<i,j'\geq i+1} g_{j',i'}\right )-\left (\sum_{i'=i,j'\geq i+1} g_{j',i'}\right ) -\left (\sum_{i'< i,j'\geq i+1} g_{j',i'}\right )\\
\delta_i-\delta_{i+1}&=\left (\sum_{i'<i,j'= i} g_{j',i'}\right )-\left (\sum_{i'=i,j'\geq i+1} g_{j',i'}\right )\\
\delta_i-\delta_{i+1}&=\left (\sum_{i'<i} g_{i,i'}\right )-\left (\sum_{j'\geq i+1} g_{j',i}\right )\\
\delta_i-\delta_{i+1}&=\left (g_{i,i}+\sum_{i'<i} g_{i,i'}\right )-\left (g_{i,i}+\sum_{j'\geq i+1} g_{j',i}\right )\\
\delta_i-\delta_{i+1}&=\left (\sum_{i'\leq i} g_{i,i'}\right )-\left (\sum_{j'\geq i} g_{j',i}\right )\\
\intertext{Since $g_{i,i+1}=n_i$, and $g_{i-1,i}=n_{i-1}$,}
\delta_i-\delta_{i+1}&=\left ((g_{i,i+1}-n_i)+\sum_{i'\leq i} g_{i,i'}\right )-\left ((g_{i-1,i}-n_{i-1})+\sum_{j'\geq i} g_{j',i}\right )\\
\delta_i-\delta_{i+1}&=\left (-n_i+\sum_{i'\leq i+1} g_{i,i'}\right )-\left (-n_{i-1}+\sum_{j'\geq i-1} g_{j',i}\right )\\
\intertext{Since $g$ represents the level graph, $g_{i,i'}=0$ if
  $i'>i+1$, or put another way, $g_{j',i}=0$ if $j'<i-1$, so}
\delta_i-\delta_{i+1}&=\left (-n_i+\sum_{i'} g_{i,i'}\right )-\left (-n_{i-1}+\sum_{j'} g_{j',i}\right )\\
\delta_i-\delta_{i+1}&=n_{i-1}-n_i
\end{align}

So, for all $0\leq i<L-1$:
\begin{align}
\delta_{i+1}-\delta_{i+2}&=n_{i}-n_{i+1}\\
n_i&=\delta_{i+1}-\delta_{i+2}+n_{i+1}\label{eqn:recursivetool}
\end{align}


For $n_0$, note that $\sum_i g_{i,0}=g_{0,0}+\delta_1+g_{L,0}$,  and
$\sum_i g_{0,i}=g_{0,0}+n_0$, so
$g_{0,0}+\delta_1+g_{L,0}=g_{0,0}+n_0$, and
$g_{L,0}=n_0-\delta_1$. This is the base case in a recursive proof
that for all $i<L$, $g_{L,0}=n_i-\delta_i$. If we wish to prove it holds
for $i+1$, then we assume it holds for $i$, or
$g_{L,0}=n_{i}-\delta_i$. By Equation~\eqref{eqn:recursivetool}, for $i<L-1$:
\begin{align}
g_{L,0}&=(\delta_{i+1}-\delta_{i+2}+n_{i+1})-\delta_i\\
&=n_{i+1}-\delta_{i+1}
\end{align}
Since $\delta_i\geq 0$, this implies that for $i<T$, $g_{L,0}\leq
n_i$, which completes the proof.
\end{proof}

\begin{lemma}\label{lem:blackwellcondition}
If Requirements~\ref{req:bounded-u} and \ref{req:stationary-distribution} hold, and $b = \Delta / (L + 1)$, then $\sum_{(i,j) \in E} \tilde{R}^{t,+}_{i,j} \pi^{t+1}_i (u^{t+1}(j) - u^{t+1}(i) - b) \le 0$.
\end{lemma}

\begin{proof}
First, consider the case where $\pi^t$ is degenerate.  Then, whenever
$\pi^{t+1}_i > 0$, we know $R^{t,+}_{i,j} = 0$ for all $(i,j) \in E$, and so our sum of interest is exactly 0.
Note that, since $f_{i,j}=\tilde{R}^+_{i,j}\pi^{t+1}_i$, what we
need to prove is:
\begin{align}
\sum_{(i,j) \in E} f_{i,j} (u^{t+1}(j) - u^{t+1}(i) - b) &\le 0\\
\left (\sum_{(i,j) \in E} f_{i,j} (u^{t+1}(j) - u^{t+1}(i) )\right )-b\sum_{(i,j) \in E} f_{i,j}  &\le 0.\label{eqn:blackwellcondition}
\end{align}
Suppose $\pi^t$ is not degenerate.  We examine
Equation~\eqref{eqn:blackwellcondition}'s two
summations.  Notice that only edges $(i,j)$ where $\pi^{t+1}_i > 0$
have $f_{i,j}\ne 0$, and by
Requirement~\ref{req:stationary-distribution}(e) this is only true if
$\Level(i) \le L$. Also, $f_{i,j}>0$ if and only if $0\leq \Level(i)\leq L$ and
$1\leq \Level(j)\leq L+1$ (because level zero has no incoming edges), so:
\begin{align}
\sum_{(i,j)\in E} f_{i,j} (u^{t+1}(j) - u^{t+1}(i))& = \sum_{(i,j)\in E} f_{i,j} u^{t+1}(j) - \sum_{(i,j)\in E} f_{i,j} u^{t+1}(i) \\
&= \sum_{\ell=1}^{L+1}\sum_{(i,j)\in E:\Level(j)=\ell} f_{i,j} u^{t+1}(j) - \sum_{\ell=0}^{L}\sum_{(i,j)\in E:\Level(i)=\ell} f_{i,j} u^{t+1}(i).
\end{align}
Renaming the dummy variables in the second term and then combining:
\begin{align}
\sum_{(i,j)\in E} f_{i,j} (u^{t+1}(j) - u^{t+1}(i))& = \sum_{\ell=1}^{L+1}\sum_{(i,j)\in E:\Level(j)=\ell} f_{i,j} u^{t+1}(j) - \sum_{\ell=0}^{L}\sum_{(j,k)\in E:\Level(j)=\ell} f_{j,k} u^{t+1}(j)\\
& = \sum_{\ell=1}^{L}\left (\sum_{(i,j)\in E:\Level(j)=\ell} f_{i,j}
  u^{t+1}(j) - \sum_{(j,k)\in E:\Level(j)=\ell} f_{j,k} u^{t+1}(j) \right
)\nonumber \\
&+\sum_{(i,j)\in E:\Level(j)=L+1} f_{i,j} u^{t+1}(j)-\sum_{(j,k)\in E:\Level(j)=0} f_{j,k} u^{t+1}(j).
\end{align}
First, we show that any term between 1 and $L$ is zero. For any
$1\leq \ell \leq L$, by summing over nodes in level $\ell$:
\begin{align}
\sum_{(i,j)\in E:\Level(j)=\ell} f_{i,j}
  u^{t+1}(j) - \sum_{(j,k)\in E:\Level(j)=\ell} f_{j,k} u^{t+1}(j)
  &=\sum_{j:\Level(j)=\ell} \left (\sum_{i:(i,j)\in E} f_{i,j}
  u^{t+1}(j) - \sum_{k:(j,k)\in E} f_{j,k} u^{t+1}(j)\right)\\
  &=\sum_{j:\Level(j)=\ell} u^{t+1}(j)\left (\sum_{i:(i,j)\in E} f_{i,j}- \sum_{k:(j,k)\in E} f_{j,k} \right).
\end{align}
By Lemma~\ref{lem:incoming-outgoing}, $\sum_{i:(i,j)\in E}
f_{i,j}=\sum_{k:(j,k)\in E} f_{j,k}$, so these terms are zero,
leaving:
\begin{align}
\sum_{(i,j)\in E} f_{i,j} (u^{t+1}(j) - u^{t+1}(i))=
&\sum_{(i,j)\in E:\Level(j)=L+1} f_{i,j} u^{t+1}(j)-\sum_{(j,k)\in
  E:\Level(j)=0} f_{j,k} u^{t+1}(j).
\end{align}
If $\Level(j)=0$, then $j=\Root$:
\begin{align}
\sum_{(i,j)\in E} f_{i,j} (u^{t+1}(j) - u^{t+1}(i))=
&\sum_{(i,j)\in E:\Level(j)=L+1} f_{i,j} u^{t+1}(j)-\sum_{(j,k)\in
  E:\Level(j)=0} f_{j,k} u^{t+1}(\Root).
\end{align}
Moreover, for any $j$, $u^{t+1}(j)-u^{t+1}(\Root)\leq \Delta$, so:
\begin{align}
\sum_{(i,j)\in E} f_{i,j} (u^{t+1}(j) - u^{t+1}(i))\leq &\sum_{(i,j)\in E:\Level(j)=L+1} f_{i,j} (u^{t+1}(\Root)+\Delta)-\sum_{(j,k)\in
  E:\Level(j)=0} f_{j,k} u^{t+1}(\Root)\\
\leq &\Delta\sum_{(i,j)\in E:\Level(j)=L+1} f_{i,j} +u^{t+1}(\Root)\left (\sum_{(i,j)\in E:\Level(j)=L+1} f_{i,j} -\sum_{(j,k)\in
  E:\Level(j)=0} f_{j,k} \right ).
\end{align}
By Lemma~\ref{lem:levelregret-dag},
Equation~\eqref{eqn:levelregrettopbottom}, the flow into level $L+1$
is less than or equal to the flow out of level 0, so the last part is
nonpositive and:
\begin{align}
\sum_{(i,j)\in E} f_{i,j} (u^{t+1}(j) - u^{t+1}(i))\leq&\Delta\sum_{(i,j)\in E:\Level(j)=L+1} f_{i,j} \label{eqn:blackwellconditionalmostdone}
\end{align}

From Lemm~\ref{lem:levelregret-dag},
Equation~\eqref{eqn:levelregrettotal}, we can show that the second term of Equation~\eqref{eqn:blackwellcondition} equals:
\begin{align}
b\sum_{(i,j)\in E} f_{i,j}
&\geq b (L+1) \sum_{(i,j): \Level(j) = L+1} f_{i,j}\label{eqn:secondpartblackwellcondition}
\end{align}
Putting Equations~\eqref{eqn:secondpartblackwellcondition} and~\eqref{eqn:blackwellconditionalmostdone}  together with the fact that $b = \Delta/(L+1)$, we get,
\begin{align}
\sum_{(i,j)\in E} f_{i,j} (u^{t+1}(j) - u^{t+1}(i) - b) 
&\le \Delta \sum_{(i,j): \Level(j) = L+1} f_{i,j} - b (L+1) \sum_{(i,j): \Level(j) = L+1} f_{i,j} \\
&\le (\Delta - b (L + 1)) \sum_{(i,j): \Level(j) = L+1} f_{i,j} = 0
\end{align}
which is what we were trying to prove.
\end{proof}



Lemma~ \ref{lem:blackwellcondition} is very close to the Blackwell
condition, but not identical, so
we sketch a quick variation on a special case of Blackwell's theorem so we can apply it
to our problem.

\begin{fact} $(a+b)^+ \le a^+ + b^+$
\end{fact}

\begin{lemma}\label{lem:plussquared} $[(a+b)^+]^2 \le (a^+ + b)^2$
\end{lemma}
\begin{proof}
\begin{enumerate}
\item If $a,b\geq 0$: $(a+b)^2\leq (a+b)^2$
\item If $a,b\leq 0$: $[(a+b)^+]^2=0\leq (a^++b)^2$.
\item If $a\geq 0, b\leq 0$: if $-b\geq a$, then $[(a+b)^+]^2=0\leq
  (a^++b)^2$, otherwise $[(a+b)^+]^2=(a+b)^2=(a^++b)^2$.
\item If $a\leq 0, b\geq 0$: then if $-a\geq b$, then$[(a+b)^+]^2=0\leq
  (a^++b)^2$, otherwise, $[(a+b)^+]^2=(a+b)^2\leq b^2=(a^++b)^2$.
\end{enumerate}

\end{proof}

\begin{fact}
If $a_{i=1\ldots n} \ge 0$ then $\sum_{i=1}^n a_i \le \sqrt{|n|
  \sum_{i=1}^n a_i^2}$.
\label{fact:sum-le-nsumsq}
\end{fact}

\begin{fact} 
$E\left[ X \right]^2 \le E\left[X^2\right]$
\label{fact:exsq-le-exsq}
\end{fact}

We restate Theorem~\ref{thm:localswap} from Section~\ref{sec:analgorothmforlocalswapregret}:
\thmlocalswap*

\begin{proof}
\begin{align}
E[R^{T}_{\mathrm{localswap}}] &= E\left[\sum_{i\in V} \left(\max_{j : (i,j) \in E}  \sum_{t=1}^{T} 1(a^t = i) (u^t(j) - u^t(i))\right)^+ \right]\\
&= E\left[\sum_{i\in V} \left(\max_{j : (i,j) \in E}  \sum_{t=1}^{T} 1(a^t = i) (u^t(j) - u^t(i) - b + b)\right)^+ \right]\\
&= E\left[\sum_{i\in V} \left(\max_{j : (i,j) \in E}  \left(\tilde{R}^{T}_{i,j} + \sum_{t=1}^{T} 1(a^t = i) b\right)\right)^+ \right]\\
&= E\left[\sum_{i\in V} \left(\left(\sum_{t=1}^{T} 1(a^t = i) b\right) + \max_{j : (i,j) \in E} \tilde{R}^{T}_{i,j}\right)^+ \right]\\
&\le E\left[\sum_{i\in V} \left(\left(\sum_{t=1}^{T} 1(a^t = i) b\right) + \left(\max_{j : (i,j) \in E} \tilde{R}^{T}_{i,j}\right)^+\right) \right]\\
&= E\left[bT + \sum_{i\in V} \max_{j : (i,j) \in E} \tilde{R}^{T,+}_{i,j}\right]\\
&\le E\left[bT + \sum_{i \in V} \sum_{j: (i,j) \in E} \tilde{R}^{T,+}_{i,j}\right] \\
&= bT + \sum_{(i,j) \in E_L}  E\left[\tilde{R}^{T,+}_{i,j}\right]\\
\intertext{By Facts~\ref{fact:sum-le-nsumsq} and~\ref{fact:exsq-le-exsq},}
&\le bT + \left(|E_L|\sum_{(i,j) \in E_L} E\left[\tilde{R}^{T,+}_{i,j}\right]^2\right)^{\frac{1}{2}} \\
&\le bT + \left(|E_L|\sum_{(i,j) \in E_L} E\left[(\tilde{R}^{T,+}_{i,j})^2\right]\right)^{\frac{1}{2}} \end{align}

We can bound the inner term as follows, using Lemma~\ref{lem:plussquared}:
\begin{align}
\sum_{(i,j) \in E_L} E\left[(\tilde{R}^{T,+}_{i,j})^2\right] 
&\le \sum_{(i,j) \in E_L}  E\left[(\tilde{R}^{T-1,+}_{i,j} + 1(a^{T} = i) (u^{T}(j) - u^{T}(i) - b))^2\right]  \\
&= \sum_{(i,j) \in E_L}  E\left[\left (\tilde{R}^{T-1,+}_{i,j} \right )^2\right] 
 + \sum_{(i,j) \in E_L} E\left[(1(a^{T} = i)(u^{T}(j) - u^{T}(i) - b))^2\right] \\
&\quad\nonumber + \sum_{(i,j) \in E_L} E\left[2 \tilde{R}^{T-1,+}_{i,j} 1(a^{T} = i) (u^{T}(j) - u^{T}(i) - b)\right] \\
&= \sum_{(i,j) \in E_L}  E\left[ \left (\tilde{R}^{T-1,+}_{i,j} \right
  )^2\right] 
 + E\left[\sum_{(i,j) \in E_L} (1(a^{T} = i)(u^{T}(j) - u^{T}(i) - b))^2\right] \\
&\quad\nonumber + 2 \!\!\!\!\!\!\!\!\sum_{a^{1,\ldots,T-1},u^{1,\ldots,T-1}}\left(E\left[\sum_{(i,j) \in E_L}
      \tilde{R}^{T-1,+}_{i,j} \pi^{T}_i (u^{T}(j) - u^{T}(i)-b) \biggr
      | a^{1, \ldots, T-1}, u^{1, \ldots, T-1}\right] \times \right.\\
&\quad\left . \Pr[a^{1, \ldots, T-1}, u^{1, \ldots, T-1}]\right )
\end{align}
By Lemma~\ref{lem:blackwellcondition}, $\sum_{(i,j) \in E_L}
      \tilde{R}^{T-1,+}_{i,j} \pi^{T}_i (u^{T}(j) - u^{T}(i)-b)\leq 0$ regardless of the previous history.
\begin{align}
\sum_{(i,j) \in E_L} E\left[(\tilde{R}^{T,+}_{i,j})^2\right]
&\le \sum_{(i,j) \in E_L}  E\left[ \left (\tilde{R}^{T-1,+}_{i,j}
  \right )^2\right] 
 + E\left[\sum_{(i,j) \in E_L} (1(a^{T} = i)(\Delta - b))^2\right] \\
&\le \sum_{(i,j) \in E_L}  E\left[\left (\tilde{R}^{T-1,+}_{i,j}
  \right )^2\right] + D (\Delta - b)^2 \\
&\le T D (\Delta - b)^2 \le T D \left(\Delta \frac{L}{L+1}\right)^2 
\end{align}

Putting these two pieces together, we get,
\begin{align}
E[R^{T}_{\mathrm{localswap}}] 
&\le bT + \left(|E_L|\sum_{(i,j) \in E_L} E\left[(R^{T,+}_{i,j})^2\right]\right)^{\frac{1}{2}} \\
&\le bT + \sqrt{|E_L| T D \left(\Delta \frac{L}{L+1}\right)^2 } \\
&\le \frac{\Delta T}{L+1} + \sqrt{TD|E_L|} \Delta \frac{L}{L+1} \\
\frac{1}{T}E[R^{T}_{\mathrm{localswap}}] &\le \frac{\Delta}{L+1} + \frac{\Delta\sqrt{D|E_L|}}{\sqrt{T}} 
\end{align}
\end{proof}

\section{Proof for Color Regret}
\label{sec:appendix:local-colored-regret}

\begin{requirement}
Let $C$ be a countable (but possibly infinite) set of colors.  The edge coloring $c : E \rightarrow C$ is such that $c(i,j) = c(i,k) \Leftrightarrow j = k$.
\label{req:coloring}
\end{requirement}



We restate Theorem~\ref{thm:color-regret-bound} from Section~\ref{sec:colorregret}:
\thmcolorregretbound*
\begin{proof}
First, we show that $\sum_{c \in C} \tilde{R}^{t,+}_c \sum_{\substack{(i,j) \in E \\ c(i,j) = c}} \pi^{t+1}_i (u^{t+1}(j)  - u^{t+1}(i) - b)) \le 0$.
\begin{align}
\sum_{c \in C} \tilde{R}^{t,+}_c \sum_{\substack{(i,j) \in E \\ c(i,j) = c}} \pi^{t+1}_i (u^{t+1}(j)  - u^{t+1}(i) - b)) &= \sum_{c \in C} \sum_{\substack{(i,j) \in E \\ c(i,j) = c}} \tilde{R}^{t,+}_c  \pi^{t+1}_i (u^{t+1}(j)  - u^{t+1}(i) - b)) \\
&= \sum_{c \in C} \sum_{\substack{(i,j) \in E \\ c(i,j) = c}}
\tilde{R}^{t,+}_{i,j}  \pi^{t+1}_i (u^{t+1}(j)  - u^{t+1}(i) - b)) 
\end{align}
By Lemma~\ref{lem:blackwellcondition}:
\begin{align}
\sum_{c \in C} \tilde{R}^{t,+}_c \sum_{\substack{(i,j) \in E \\ c(i,j) = c}} \pi^{t+1}_i (u^{t+1}(j)  - u^{t+1}(i) - b))&= \sum_{(i,j)\in E} \tilde{R}^{t,+}_{i,j}  \pi^{t+1}_i (u^{t+1}(j)  - u^{t+1}(i) - b)) \le 0
\end{align}

Now we can bound our quantity of interest.
{\allowdisplaybreaks\begin{align}
E[R^T_{\mathrm{localcolor}}] &= E\left[\sum_{c\in C} \left(\sum_{\substack{(i,j) \in E \\ c(i,j) = c}} \sum_{t=1}^T 1(a^t = i) (u^t(j) - u^t(i))\right)^+ \right]\\
&= E\left[\sum_{c\in C} \left(\sum_{\substack{(i,j) \in E \\ c(i,j) = c}} \sum_{t=1}^T 1(a^t = i) (u^t(j) - u^t(i) - b + b)\right)^+ \right]\\
&= E\left[\sum_{c\in C} \left(\tilde{R}^T_c + \left(\sum_{\substack{(i,j) \in E \\ c(i,j) = c}} \sum_{t=1}^T 1(a^t = i) b \right)\right)^+ \right]\\
&\le E\left[\sum_{c\in C} \left(\tilde{R}^{T,+}_c + \left(\sum_{\substack{(i,j) \in E \\ c(i,j) = c}} \sum_{t=1}^T 1(a^t = i) b \right)\right) \right]\\
&= E\left[\sum_{c\in C} \sum_{\substack{(i,j) \in E \\ c(i,j) = c}} \sum_{t=1}^T 1(a^t = i) b + \sum_{c\in C}  R^{T,+}_c\right]\\
&\le E\left[bTD + \sum_{c\in C}  \tilde{R}^{T,+}_c \right] \\
&= bTD + \sum_{c\in C_L}  E\left[\tilde{R}^{T,+}_c \right] \\
&\le bTD + \left(|C_L| \sum_{c\in C_L} E\left[\tilde{R}^{T,+}_c \right]^2\right)^{\frac{1}{2}} \\
&\le bTD + \left(|C_L| \sum_{c\in C_L} E\left[(\tilde{R}^{T,+}_c)^2 \right]\right)^{\frac{1}{2}}
\end{align}}
We can bound the inner term as follows,
{\allowdisplaybreaks\begin{align}
\sum_{c \in C_L} E\left[(\tilde{R}^{T,+}_c)^2\right] 
&\le \sum_{c \in C_L} E\left[ (\tilde{R}^T_c)^2\right] \\
&= \sum_{c \in C_L}  E\left[\left(\tilde{R}^{T-1}_c + 
   \sum_{\substack{(i,j) \in E \\ c(i,j) = c}} 1(a^{T} = i) (u^T(j) - u^T(i) - b)\right)^2\right]  \\
&= \sum_{c \in C_L}  E\left[(\tilde{R}^{T-1}_c)^2\right]
 + \sum_{c \in C_L} E\left[\left(\sum_{\substack{(i,j) \in E \\ c(i,j) = c}} 1(a^{T} = i) (u^T(j) - u^T(i) - b)\right)^2\right] \\
&\quad\nonumber + \sum_{c\in C_L} E\left[2 \tilde{R}^{T-1}_c \sum_{\substack{(i,j)\in E\\ c(i,j) = c}} 1(a^{T} = i) (u^T(j) - u^T(i) - b)\right] \\
&= \sum_{c \in C_L}  E\left[(\tilde{R}^{T-1}_c)^2\right]
 + \sum_{c \in C_L} E\left[\left(\sum_{\substack{(i,j) \in E \\ c(i,j) = c}} 1(a^{T} = i) (u^T(j) - u^T(i) - b)\right)^2\right] \\
&\quad\nonumber + 2 \sum_{c\in C_L} \tilde{R}^{T-1}_c \sum_{\substack{(i,j)\in E\\ c(i,j) = c}} \pi^{T}_i (u^T(j) - u^T(i) - b) \\
&\le \sum_{c \in C_L}  E\left[(\tilde{R}^{T-1}_c)^2\right]
 + \sum_{c \in C_L} E\left[\left(\sum_{\substack{(i,j) \in E \\ c(i,j) = c}} 1(a^{T} = i) (u^T(j) - u^T(i) - b)\right)^2\right] \\
&\le \sum_{c \in C_L}  E\left[(\tilde{R}^{T-1}_c)^2\right]
 + (\Delta - b)^2 \sum_{c \in C_L} E\left[\left(\sum_{\substack{(i,j) \in E \\ c(i,j) = c}} 1(a^{T} = i) \right)^2\right] 
\end{align}}

Because only one action is taken, and for each color only one edge originating at an
action can have that color, $\sum_{\substack{(i,j) \in E \\ c(i,j) = c}}
1(a^{T} = i)\in \{0,1\}$:
\begin{align}
\sum_{c \in C_L} E\left[(\tilde{R}^{T,+}_c)^2\right]
&\le\sum_{c \in C_L}  E\left[(\tilde{R}^{T-1}_c)^2\right]
 + (\Delta - b)^2 \sum_{c \in C_L} E\left[\left(\sum_{\substack{(i,j) \in E \\ c(i,j) = c}} 1(a^{T} = i) \right)\right] \\
&= \sum_{c \in C_L}  E\left[(\tilde{R}^{T-1}_c)^2\right]
 + D (\Delta - b)^2 \\
&\le TD (\Delta - b)^2 \le TD \left(\Delta \frac{L}{L+1}\right)^2
\end{align}

Putting these two pieces together, we get,
\begin{align}
E[R^T_{\mathrm{colorswap}}] 
&\le bTD + \left(|C_L|\sum_{c \in C_L} E\left[(\tilde{R}^{T,+}_c)^2\right]\right)^{\frac{1}{2}} \\
&\le bTD + \sqrt{|C_L| T D \left(\Delta \frac{L}{L+1}\right)^2 } \\
&\le \frac{\Delta D T}{L+1} + \sqrt{TD|E_L|} \Delta \frac{L}{L+1} \\
\frac{1}{T}E[R^T_{\mathrm{colorswap}}] &\le \frac{\Delta D}{L+1} + \frac{\Delta\sqrt{D|C_L|}}{\sqrt{T}} 
\end{align}
\end{proof}

\section{Decision Tree Graphs}
\label{sec:appendix:dtree}

A decision tree is a representation of a hypothesis. Given an instance
space where there are a finite number of binary features, a decision
tree can represent an arbitrary hypothesis.
We describe decision trees recursively: the simplest trees are
leaves, which represent constant functions. More complex trees have
two subtrees, and a root node labeled with a variable. A subtree
cannot have a variable that is referred to in the root.

We define $T_k(S)$ recursively, where $T_k(S)$ will be the set of trees of depth $k$
or less
over the variable set $S$.
Define the set $T_0(S)=\{\TT,\FA\}$. Define $T_k(S)$ such that:
\begin{align}
T_k(S)&=T_{k-1}(S) \bigcup_{s\in S} (\{s\}\times T_{k-1}(S\backslash
\{s\})\times T_{k-1}(S\backslash \{s\}))
\end{align}
Define $T^*(S)=T_{|S|}(S)$ to be the set of all decision
trees over the variables $S$. Three example decision
trees in $T^*(\{x_1,x_2\})$ are $(x_1,\TT,\FA)$, $\TT$, and $(x_1,(x_2,\TT,\FA),\FA)$. 
Suppose we have an example $x$, mapping variables to
$\{\TT,\FA\}$. For any tree $t$, we can recursively define $t(x)$:
\begin{enumerate}
\item If $t\in T_0,$ then $t(x)=t$. 
\item If $t\in T_k$ and $x(t_1)=\TT$, then $t(x)=t_2(x)$.
\item If $t\in T_k$ and $x(t_1)=\FA$, then $t(x)=t_3(x)$.
\end{enumerate}

Define $P=\{p\in (S\times \{\TT,\FA\})^{|S|}:\forall i\ne j,
p_{i,1}\ne p_{j,1}\}$ to be the paths in the trees without repeating variables.
We can talk about whether a path is in a tree. Define $V_p(t)$ to be a function from $T^*$ to $S\cup
\{\TT,\FA,\emptyset\}$, where $V_{p}(t)=\emptyset$ if the path $p$ is
not present in the tree, and otherwise $V_p(t)$ is the value of the node
at the end of the path.  Formally, 
\begin{align}
V_{\emptyset}(t)&=\ifotherwise{t}{t\in T_0}{t_1}\\
V_{(v,l)\circ p}(t)&=\begin{mazcases}
\emptyset&\mbox{ if }t\in T_0 \mbox{ or }t_1\ne v\\
V_{p}(t_2)&\mbox{ if }t\notin T_0\mbox{ and }t_1=v\mbox{ and } l=\TT\\
V_{p}(t_3)&\mbox{ if }t\notin T_0\mbox{ and }t_1=v\mbox{ and } l=\FA
\end{mazcases}
\end{align}

Given a path $p\in P$, a tree $t'\in T^*$,define $R_{p,t'}(t)$ to replace the tree at $p$ with
$t'$ if $V_p(t)\ne \emptyset$. Formally:
\begin{align}
R_{\emptyset,t'}(t)&=t'\\
R_{(v,l)\circ p,t'}(t)&=\begin{mazcases}
t&\mbox{ if }t\in T_0 \mbox{ or }t_1\ne v\\
(t_1,R_{p,t'}(t_2),t_3)&\mbox{ if }t\notin T_0\mbox{ and }t_1=v\mbox{ and } l=\TT\\
(t_1,t_2,R_{p}(t_3))&\mbox{ if }t\notin T_0\mbox{ and }t_1=v\mbox{ and } l=\FA
\end{mazcases}
\end{align}

Consider the following operations on decision trees:

\begin{denumerate}
\item $ReplaceWithNode(p,v,l_1,l_2)=R_{p,(v,l_1,l_2)}$ (where it applies): If there exists a node or leaf at path
  $p$, replace it with a decision stump with variable $v$, with label
  $l_1$ on the true branch, and label $l_2$ on the false branch,
  but~\textbf{only if $V_{p}(t)\ne v$}.
\item $ReplaceWithLeaf(p,l_1)=R_{p,l_1}$: If there exists a node or leaf at path
  $p$, replace it with a leaf $l_1$.
\end{denumerate}

These operations create the edges between trees: we will determine how
to color them later. Because $ReplaceWithNode$ is a more complex
operation, an edge created by $ReplaceWithNode$ will have length 1.1, whereas
$ReplaceWithLeaf$ will have length 1.0. This weighting is important:
otherwise, consider the following sequence of trees:
\begin{eqnarray*}
&(X,\TT,\FA)&\\
&(\FA)&\\
&(X,\FA,\TT)&
\end{eqnarray*}
If splitting was the same length as changing leaves, this bizarre path
would be a shortest path between $(X,\TT,\FA)$ and
$(X,\FA,\TT)$. In general, when designing this distance function over trees, a
critical concern was whether unnecessary reconstruction would be on a
shortest path. For example, a shortest path from
$(X,(Y,\TT,\FA),(Z,\FA,\TT))$ to $(X,(Y,\FA,\TT),(Z,\TT,\FA))$ could 
pass through $\FA,(X,\TT,\FA),(X,(Y,\FA,\TT),\FA)$. But, since
replacing something with a decision tree costs slightly more than
changing a leaf, we avoid this.

More generally, if the decision about whether or not an
edge is on the shortest path
can be made locally, then this reduces the number of colors required. Thus, massively
reconstructing the root because the leaves are wrong is not only
counterintuitive, it makes the algorithm slower and more complex.

We first hypothesize a shortest path distance function between trees based on these
operations, and then we will prove it satisfies the above
operations. Note that this function is not symmetric, because the
shortest path distance function on a directed graph is not always symmetric.

Given two decision trees $A$ and $B$, a decision node $a$ in $A$ and a
decision node $b$ are in structural agreement if they are on the same path $p$, and they are
labeled with the same variable. A decision node in $B$ that does not
agree with a decision node in $A$ is in structural disagreement with $A$.
Given a leaf in $B$ that has a parent that is in structural agreement
with $A$, if the leaf is not present in $A$, it is in leaf
disagreement with $A$.

Define $d^*_s(A,B)$ to be the structural disagreement distance between
$A$ and $B$, the number of nodes in $B$ that are in structural
disagreement with $A$. Define $d^*_l(A,B)$ to be the leaf disagreement
distance between $A$ and $B$, the number of leaves in $B$ in
disagreement with $A$. Define $d^*(A,B)=1.1d^*_s(A,B)+d^*_l(A,B)$.

Intuitively, this distance represents the fact that an example shortest path
from $A$ to $B$ can be generated by first fixing all label
disagreements between $A$ and $B$, and then applying $ReplaceWithNode$
to create every node in $B$ that is in structural disagreement with
$A$ (correctly labeling leaves where appropriate).

\newcommand{\ZZ}{\mathbf{Z}}

\begin{fact}If $d:V\times V\rightarrow \ZZ^+$ is the shortest distance function on
a completely connected directed
graph $(V,E)$, then for any $i,j\in V$ where $(i,j) \notin E$, there exists a $k$ such that
$(i,k)\in E$ and $d(i,j)=d(i,k)+d(k,j)$.\label{fct:distancefct}
\end{fact}

\begin{theorem}\label{thm:graphdistancetheorem}$d^*:V\times V\rightarrow \ZZ^+$ corresponds to the
  shortest distance function on a
  completely connected directed graph $(V,E)$ if there exists a
  $\Delta>0$ and a $\delta=\Delta/2$ such that the following properties hold:
\begin{denumerate}
\item For all $a,b\in V$, $d^*(a,b)=0$ iff $a=b$.\label{enm:zerodistance}
\item For all $a,b\in V$, $d^*(a,b)>\delta$ iff $a\ne b$.\label{enm:mindistance}
\item For all $a,b\in V$, if $a\ne b$ there exists a $c\in V$ such that
  $d^*(a,c)\leq \Delta$ and $d^*(a,b)\geq d^*(a,c)+d^*(c,b)$.\label{enm:equal}
\item For all $a,b,c\in V$, if $d^*(a,c)\leq \Delta$, then $d^*(a,b)\leq
  d^*(a,c)+d^*(c,b)$. \label{enm:lessthan}
\end{denumerate}
\end{theorem}

\begin{proof}
Observe that the graph $(V,E)$ with edges $E=\{(i,j)\in
V^2:d^*(i,j)\leq \Delta\}$ where the weight of an edge $(i,j)\in E$ is
$d^*(i,j)$, is a good candidate for the graph under consideration. We prove this in two
steps. We first prove by induction that $d(i,j)\leq d^*(i,j)$. Then,
leveraging this, we prove by induction that $d(i,j)=d^*(i,j)$.

First, we prove that if $d^*(i,j)\leq \Delta$, then
$d(i,j)=d^*(i,j)$. First, observe that if $d^*(i,j)=0$, then $i=j$, so
$d(i,j)=0$. Secondly, if $d^*(i,j)\in (0,\Delta]$, then there exists
an edge $(i,j)\in E$ so $d(i,j)\leq d^*(i,j)$. Since each edge is
larger than $\Delta/2$, for any path of length
2 or greater, the length is larger than $\Delta$, so only a direct
path can be less than or equal to $\Delta$. This establishes that
there is no path between $i$ and $j$ shorter than the direct
edge. 

For any nonnegative integer $k$, define $P(k)$ to be the property that
for any $i,j\in V$, if the distance $d^*(i,j)\leq k\delta$, the shortest distance between two
vertices in this graph $d(i,j)$ is less than or equal to
$d^*(i,j)$. This holds for $P(0)$, $P(1)$, and $P(2)$ because of the paragraph
above. Now, suppose that $P(k)$ holds for $k\geq 2$, we need to establish it holds
for $P(k+1)$. Consider some pair $(i,j)\in V$ where $d^*(i,j)\in (k\delta,(k+1)\delta]$, then $i\ne j$, and by
condition~\ref{enm:equal}, there exists a $k$ where $d^*(i,k)\leq \Delta$ and
$d^*(i,j)\geq d^*(i,k)+d^*(k,j)$. Since $d^*(i,j)\leq (k+1)\delta$ and
$d^*(i,j)>\delta$, $d^*(k,j)<k\delta$, so $d^*(k,j)=d(k,j)$. From the
paragraph above, $d(i,k)=d^*(i,k)$, so $d^*(i,j)\geq d(i,k)+d(k,j)$,
and by the triangle inequality on $d$, $d^*(i,j)\geq d(i,j)$.

Thus, since for all $(i,j)\in V$ there exists a $k$ where
$d^*(i,j)\leq k\delta$, for all $(i,j)\in V$, $d(i,j)\leq d^*(i,j)$.

Next, we prove that if $d(i,j)\leq \Delta$, then
$d(i,j)=d^*(i,j)$. First, observe that if $d(i,j)=0$, then $i=j$, so
$d^*(i,j)=0$. Secondly, if $(i,j)\notin E$, then the distance between
$i$ and $j$ must be greater than $\Delta$, because each edge is larger
than $\Delta/2$. Therefore, if $d(i,j)\in (0,\Delta]$ there is a
direct edge between $i$ and $j$ with distance $d^*(i,j)$, so
$d^*(i,j)\leq \Delta$, and so by the second paragraph
$d(i,j)=d^*(i,j)$.

Define $Q(k)$ to be the property for any $(i,j)\in V$, if $d(i,j)\leq k\delta$ then
$d(i,j)=d^*(i,j)$.
$Q(0)$, $Q(1)$ and $Q(2)$ hold from the above paragraph.
Now, suppose that $Q(k)$ holds for some $k\geq 2$, we need to
establish the property for $Q(k+1)$. Consider some pair $(i,j)\in V$ where $d(i,j)\in (k\delta,(k+1)\delta]$, then $i\ne j$, and by
condition~\ref{fct:distancefct}, there exists a $k$ where there exists
an edge from $i$ to $k$ and $d(i,j)=d(i,k)+d(k,j)$. Since there exists
an edge $(i,k)$, then $d(i,k)\leq \Delta$ and $d(i,k)=d^*(i,k)>\delta$. Thus,
$d(k,j)\leq \delta (k+1)-\delta\leq  \delta k$. so
$d(k,j)=d^*(k,j)$. Moreover, by condition~\ref{enm:lessthan}, $d^*(i,j)\leq
d^*(i,k)+d^*(k,j)=d(i,j)$. Thus, since we know that $d^*(i,j)\geq
d(i,j)$, then $d^*(i,j)=d(i,j)$.

Therefore, since $d^*(i,j)=d(i,j)$, and $d$ is the shortest distance
for graph $(V,E)$, then $d^*(i,j)$ is a shortest distance function for
a weighted graph.
\end{proof}

\begin{lemma}\label{lem:distancebig}
For the decision tree metric $d^*$ above, for any two trees
  $A,B$ where $A\ne B$, there exists a tree $C$ such that
  $d^*(A,C)\leq 1.1$ and
  $d^*(A,B)\geq d^*(A,C)+d^*(C,B)$.
\end{lemma}

\begin{proof}
If $B$ has a leaf at the root, then set $C=B$.

Suppose that, given $A$ and $B$, there is label
disagreement. Find the a node with label disagreement, and
correct all the labels in $A$ to form $C$. This reduces the number of nodes with
label disagreement by one, and the decision node
disagreement stays the same.

Suppose that, given $A$ and $B$, there no label disagreement, but
there is structural disagreement. Then
select a node $d$ which has decision node disagreement. Define $C$ to
be a tree where we replace node $d$ with the corresponding node in
tree $B$, with leaves that agree with the children of $d$ if $d$ has
children, and arbitrary otherwise. This reduces the structural
disagreement by one. It does not increase the label disagreement,
because if $d$ has children with labels in $B$, it has those same
children in $C$.

Finally, if $A$ and $B$ have no label disagreement or structural disagreement, then they are the same
tree and have distance 0.
\end{proof}

Before proving a lower bound, we focus on a particular case. Namely,
that changing a correct decision node of a tree to have the wrong variable cannot
decrease the distance.

\begin{lemma}\label{lem:leavesandnodes}
Given two trees $A$ and $B$ and a subtree $S$ in $B$, if $n_S$ is the
number of nodes in agreement with $B$ in the subtree $S$, and
$l_S$ is the number of leaves in disagreement with $A$ in $S$, then $l_S\leq n_S+1$.
\end{lemma}

\begin{proof}
We prove this by recursion on the size of the subtree $S$ in $B$. If
$S$ is of size 1, then $S$ is a
leaf in $B$, then $n_S=0$ and $l_S\leq 1$, so the result holds. 
Suppose we have proven this for all subtrees $S'$ of size less than
$S$. If $S$
is rooted at a node in disagreement, then $n_s=0$ and $l_S=0$, and the
result holds (we don't need induction for this case). If $S$ is rooted
at a node $x$ in agreement, then define $S_{\TT}$ to be the subtree
of the node down the edge labeled $\TT$ leaving $x$, and define $S_{\FA}$ to
be the subtree down the edge labeled $\FA$ leaving $x$. $|S_{\TT}|<|S|$
and $|S_{\FA}|<|S|$, so by induction $l_{S_{\TT}}\leq n_{S_{\TT}}+1$ and
$l_{S_{\FA}}\leq n_{S_{\FA}}+1$. Since $x$ is a node in agreement,
$l_{S}=l_{S_{\TT}}+l_{S_{\FA}}$, and therefore:
\begin{align}
l_{S}&\leq n_{S_{\TT}}+n_{S_{\FA}}+1+1\\
\end{align}
Again, since $x$ is a node in agreement,
$n_{S_{\TT}}+n_{S_{\FA}}+1=n_S$, so:
\begin{align}
l_{S}&\leq n_{S}+1.
\end{align}
\end{proof}

We will use this fact in several places in the resulting proofs.

\begin{lemma}\label{lem:nocloserbackwards}
Given two trees $A$ and $B$ which agree on node $y$, if you change
$y$ in $A$ to a node $x$ or leaf to create $C$, then $d^*(A,B)<d^*(C,B)+1$.
\end{lemma}

\begin{proof}
If $S$ is the subtree rooted at $y$ in $B$, then
$d^*_s(A,B)+n_S=d^*_s(C,B)$ and $d^*_l(A,B)-l_S=d^*_l(C,B)$. By definition,
$d^*(A,B)=d^*(C,B)+1.1n_S-l_S$. Since $y$
is in agreement, $n_S\geq 1$. By Lemma~\ref{lem:leavesandnodes}, we
know that $l_S\leq n_S+1$, so 
\begin{align}
d^*(A,B)&=d^*(C,B)+1.1n_S-(n_S+1)\\
d^*(A,B)&=d^*(C,B)+0.1n_S+1
\end{align}
Since $n_s\geq 1$, $0.1n_s\geq 0.1>0$, so:
\begin{align}
d^*(A,B)&<d^*(C,B)+1
\end{align}
\end{proof}


\begin{lemma}\label{lem:distancesmall}
For the decision tree metric $d^*$ above, for any two trees
  $A,B$ where $A\ne B$, then for any $C$ such that $d^*(A,C) \leq \Delta$,
  $d^*(A,B)\leq d^*(A,C)+d^*(C,B)$.
\end{lemma}

\begin{proof}
First, observe that $C$ has ``one'' change from $A$, which can be
that:

\begin{denumerate}
\item $C$ has a decision node splitting on variable $x$ where $A$ had a
  decision node splitting on variable $y$.
\item $C$ has a decision node splitting on variable $x$ where $A$ had a
  leaf $l$.
\item $A$ has a node $x$ that was changed to a leaf.
\item $C$ has a leaf where $A$ had a node.
\end{denumerate}

In the first case, there is a question of whether or not the decision
node $y$ exists in $B$. If so, then the structural disagreement has
been reduced by one. However, the leaf disagreement is unchanged or
increased by one, so   $d^*(A,B)\leq 1.1+d^*(C,B)=d^*(A,C)+d^*(C,B)$.
If $y$ is not in $B$, and $x$ is not in $B$, then
$d^*(A,B)=d^*(C,B)< 1.1+d^*(C,B)=d^*(A,C)+d^*(C,B)$. If $y$ is in $B$, by Lemma~\ref{lem:nocloserbackwards},
then $d^*(A,B)< d^*(C,B)+1<1.1+d^*(C,B)=d^*(A,C)+d^*(C,B)$.

For the second case, if the new node in $C$ agrees with $B$, then
$d^*(A,B)=1.1+d^*(C,B)$. If the leaf in $A$ agreed with $B$, then 
$d^*(A,B)= d^*(C,B)-1<1.1+d^*(C,B)=d^*(A,C)+d^*(C,B)$. If the leaf in
$A$ disagreed with $B$ and the new node in $C$ disagrees with $B$,
then $d^*(A,B)=d^*(C,B)<1.1+d^*(C,B)=d^*(A,C)+d^*(C,B)$.

For the third case, if the new leaf in $C$ agrees with $B$, then
$d^*(A,B)=1+d^*(C,B)=d^*(A,C)+d^*(C,B)$. If the node in $A$ agreed
with $B$, then by Lemma~\ref{lem:nocloserbackwards},
$d^*(A,C)<d^*(C,B)+1=d^*(A,C)+d^*(C,B)$. If the node in $A$ disagreed
with $B$, and the new leaf in $C$ disagrees with $B$, then
$d^*(A,B)=d^*(C,B)<1+d^*(C,B)=d^*(A,C)+d^*(C,B)$.

Finally, for the fourth case, if the new leaf in $C$ agrees with $B$,
then $d^*(A,B)=1+d^*(C,B)=d^*(A,C)+d^*(C,B)$. If the leaf in $A$ agreed
with $B$, then by Lemma~\ref{lem:nocloserbackwards},
$d^*(A,C)<d^*(C,B)+1=d^*(A,C)+d^*(C,B)$. If the leaf in $A$ disagreed
with $B$, and the new leaf in $C$ disagrees with $B$, then there was
no change, and this is an illegal transition.
\end{proof}

\begin{theorem}\label{thm:distanceright}
The distance $d^*$ as defined above is the distance function for a
graph.
\end{theorem}

\begin{proof}
In order to prove this, we use
Theorem~\ref{thm:graphdistancetheorem}. First $\Delta=1.1$, and
$\delta=0.55$. 

Observe that by the definition of $d^*$, if two trees
are equal, there is no
disagreement, and there is zero distance. Secondly, by the definition
of $d^*$, if there is any difference between two trees $A$ and $B$, there will be
disagreement, and $d^*(A,B)\geq 1$. Thus,
Condition~\ref{enm:zerodistance} and Condition~\ref{enm:mindistance} are satisfied.

Now, by Lemma~\ref{lem:distancebig}, Condition~\ref{enm:equal}
is satistfied. By Lemma~\ref{lem:distancesmall}, Condition~\ref{enm:lessthan}
is satisfied.

\end{proof}

In the graph generated from $d^*$, note that a single label
disagreement or a single decision node disagreement results in an edge.

Now, we have to derive colors.

\begin{denumerate}
\item $ReplaceWithNode(p,v,l_1,l_2)$: The path, the variable, and the
  labels form the color. Note that if the tree already has a decision node with label
  $v$ at path $p$, this transition is illegal.
\item $ReplaceWithLeaf(p,l_1)$: The path and the leaf form the color.
\end{denumerate}

\begin{lemma}
$ReplaceWithNode(p,v,l_1,l_2)$ is on the shortest path to $B$ if
\begin{denumerate}
\item it can be applied to the current tree
\item the variable $v$ is at the path $p$ in $B$.
\item A leaf with the label $\lnot l_1$ is not at the path $p\circ (v,\TT)$ in $B$,
\item A leaf with the label $\lnot l_2$ is not at the path $p\circ (v,\FA)$ in $B$.
\end{denumerate}
If these rules do not apply, it is not on the shortest path.
\end{lemma}

\begin{proof}

Suppose that $A$ is our current tree. Suppose that $C=R_{p,(v,l_1,l_2)}(A)$.

First, we establish that if the conditions are satisfied, the edge is
on the shortest path. Note that if $v$ is at the path $p$ in $B$, and there is a leaf or
another decision node at path $p$ in $A$, then $v$ is in structural
disagreement. Therefore, when we replace that node with $v$, we reduce
the structural disagreement. However, we must be careful not to
increase leaf disagreement. If, for any nodes of $v$ in $B$, they are
corrected in $A$, then leaf disagreement will not
increase. Therefore, by reducing the structural disagreement by 1, we
reduce the distance by 1.1, at a cost of 1.1, meaning the edge is on
the shortest path.

Secondly, we can go through the conditions one by one to realize any
violated condition is sufficient. Regarding the first condition: if the operation cannot be
applied to the current tree, then by definition it is not on the
shortest path.

Regarding the second condition: if the variable $v$ is not on path $p$
in $B$, but $A$ and $B$ are in agreement at the path $p$, then changing
the variable to $v$ will not decrease the distance sufficiently, by
Lemma~\ref{lem:nocloserbackwards}, so it is not on the shortest
path. Secondly, if $A$ does not agee with $B$ on path $p$, then
$d^*(A,B)=d^*(C,B)$, and thus $C$ is not on the shortest
path. 

Regard the third and fourth conditions. If the variable $v$ is on the
path $p$ in $B$, but there is some leaf that is
a child of $v$ in $B$ that is set incorrectly, then the structural
distance is decreased, but the leaf disagreement is increased, so
$d^*(A,B)=d^*(C,B)+0.1$. 
\end{proof}

\begin{lemma}
$ReplaceWithLeaf(p,l_1)$ is on the shortest path to $B$ if it applies to the
current tree, and if the leaf
$l_1$ is at $p$ in $B$. If these rules do not apply, it is not on the shortest path.
\end{lemma}

\begin{proof}
Suppose that $A$ is the initial tree, and $C=R_{p,l_1}(A)$. If the edge applies, and there
is the wrong label or a decision node at $p$, then the label is in
disagreement in $A$, but not in $C$. There are no other changes, so
$d^*(A,B)=d^*(C,B)+1=d^*(A,C)+d^*(C,B)$, and therefore the edge is on
a shortest path.

On the other hand, if there is no leaf at $p$ in $B$, or the leaf has
another label, then this is not the shortest path.

First of all, if the operator does not apply to $A$, it cannot be on
the shortest path.

If the label $V_p(B)\ne l_1$, but $A$ and $B$ are in agreement at the path $p$, then by
Lemma~\ref{lem:nocloserbackwards}
$d^*(A,B)<d^*(C,B)+1=d^*(A,C)+d^*(C,B)$.  If $V_p(B)\ne l_1$, and $A$
and $B$ are not in agreement at the path $p$, then $d^*(A,B)=d^*(C,B)<d^*(C,B)+1$.
\end{proof}

Thus, we have established our coloring works for decision trees.


\fi

\end{document}